\documentclass{article}
\usepackage[a4paper, total={6in, 8.3in}]{geometry}

\newcommand{\RNum}[1]{\uppercase\expandafter{\romannumeral #1\relax}}

\usepackage{makecell}
\usepackage{multirow}

\usepackage{mathtools}
\usepackage{amssymb}
\usepackage{amsthm}
\usepackage{amsmath}
\usepackage{soul}
\usepackage{xcolor}

\usepackage{hyperref}

\usepackage{algpseudocode}
\usepackage[ruled]{algorithm2e}
\SetKwRepeat{Do}{do}{while}

\algnewcommand{\comment}[1]{\Comment{{#1}}}

\theoremstyle{remark}

\newtheorem{theorem} {Theorem}
\newtheorem{lemma} {Lemma}
\newtheorem{definition} {Definition}

\newtheorem{observation} {Observation}
\newtheorem{assumption} {Assumption}

\DeclareMathOperator*{\argmin}{argmin}

\def\vz{{\textbf{0}}}
\def\g{{\mathbf{g}}}

\def\x{{\mathbf{x}}}

\def\u{{\mathbf{u}}}
\def\v{{\mathbf{v}}}
\def\z{{\mathbf{z}}}
\def\w{{\mathbf{w}}}
\def\y{{\mathbf{y}}}

\def\X{{\mathbf{X}}}

\newcommand{\mT}{\mathcal{T}}

\newcommand{\R}{\mathcal{R}}
\newcommand{\mP}{\mathcal{P}}

\newcommand{\mK}{\mathcal{K}}
\newcommand{\mS}{\mathcal{S}}
\newcommand{\ball}{\mathcal{B}}

\newcommand{\E}{\mathbb{E}}

\newcommand{\dist}{\textrm{dist}}

\newcommand{\reals}{\mathbb{R}}

\title{New Projection-free Algorithms for Online Convex Optimization with Adaptive Regret Guarantees\thanks{This version subsumes the version published in the Conference on Learning Theory (COLT) 2022 and fixes an error in the proof of Theorem 10 (convergence for strongly convex losses) in the COLT version. The new regret bound is worse by a logarithmic factor.}}
\author{
  Dan Garber\\
  {\small Technion - Israel Institute of Technology}\\
  {\small \texttt{dangar@technion.ac.il}}
  \and
  Ben Kretzu\\
  {\small Technion - Israel Institute of Technology}\\
  {\small \texttt{benkretzu@campus.technion.ac.il}}
}
\date{}
\begin{document}

\maketitle

\begin{abstract}%
    We present new efficient \textit{projection-free} algorithms for online convex optimization (OCO), where by projection-free we refer to algorithms that avoid computing orthogonal projections onto the feasible set,  and instead relay on different and potentially much more efficient oracles. While most state-of-the-art projection-free algorithms are based on the \textit{follow-the-leader} framework, our algorithms are fundamentally different and are based on the \textit{online gradient descent} algorithm with a novel and efficient approach to computing  so-called \textit{infeasible projections}.  As a consequence, we obtain the first projection-free algorithms which naturally yield \textit{adaptive regret} guarantees, i.e., regret bounds that hold w.r.t. any sub-interval of the sequence. 
    Concretely, when assuming the availability of a linear optimization oracle (LOO) for the feasible set, on a sequence of length $T$, our algorithms guarantee $O(T^{3/4})$ adaptive regret and  $O(T^{3/4})$ adaptive expected regret, for the full-information and bandit settings, respectively, using only $O(T)$ calls to the LOO. These bounds match the current state-of-the-art regret bounds for LOO-based projection-free OCO, which are \textit{not adaptive}. 
We also consider a new natural setting in which the feasible set is accessible through a separation oracle.     
    We present algorithms which, using overall $O(T)$ calls to the separation oracle, guarantee  $O(\sqrt{T})$ adaptive regret and $O(T^{3/4})$ adaptive expected regret for the full-information and bandit settings, respectively. 
\end{abstract}

\section{Introduction} \label{sec:intro}
In this paper we consider the problem of Online Convex Optimization (OCO) \cite{HazanBook, Shalev12} with a particular focus on so-called \textit{projection-free} algorithms. Such algorithms are motivated by high-dimensional problems in which the feasible decision set admits a non-trivial structure and thus, computing orthogonal projections onto it, as required by standard methods, is often computationally prohibitive. Instead, projection-free methods access the decision set through a conceptually simpler oracle which in many cases of interest admits a much more efficient implementation than that of an orthogonal projection oracle. Indeed, for this reason such algorithms have drawn significant interest in recent years, see for instance \cite{Hazan12, garber2013playing, chen2019projection, garber2020improved, kretzu2021revisiting, hazan2020faster, Levy19, wan2021projection, ene2021projection, pmlr-v80-chen18c, zhang2017projection}.

Let us introduce some formalism before moving on. Throughout the paper we assume without loosing much generality that the underlying vector space is $\reals^n$. We recall that in OCO, a decision maker (DM) is required throughout $T$ iterations (we will assume throughout that $T$ is known in advanced for ease of presentation), to pick on each iteration $t\in[T]$, a decision in the form of a point $\x_t$ from some  fixed convex and compact decision set $\mK\subset\reals^n$. After choosing $\x_t\in\mK$, the DM incurs a loss given by $f_t(\x_t)$, where $f_t:\reals^n\rightarrow\reals$ is convex\footnote{In fact, it suffices that $f_t$ is convex on a certain Euclidean ball containing  the set $\mK$.}. We will make the standard distinction between the \textit{full-information} setting, in which after incurring the loss, the DM gets to observe the loss function $f_t(\cdot)$, and the \textit{bandit} setting, in which the DM only learns the value $f_t(\x_t)$. In the full-information setting we shall assume that the sequence of losses $f_1,\dots,f_T$ is arbitrary, and may even depend on the plays of the DM, while in the bandit setting we shall make a standard simplifying assumption that $f_1,\dots,f_T$ are chosen in \textit{oblivious} fashion, i.e., before the DM has made his first step (and thus are in particular independent of any randomness introduced by the DM). We recall that the standard measure of performance in OCO, which is also the objective that the DM usually strives to minimize, is the \textit{regret} (or its expectation in the bandit setting) which, given the entire history $\{\x_t,f_t\}_{t=1}^T$, is given by
\begin{align}\label{eq:regret}
\textrm{Regret} = \sum_{t=1}^Tf_t(\x_t) - \min_{\x\in\mK}\sum_{t=1}^Tf_t(\x).
\end{align}
Most projection-free OCO algorithms are based on a combination of the \textit{Follow-The-Leader} (FTL) meta-algorithm, and in particular its deterministically regularized variant known as \textit{Regularized-Follow-The-Leader} (RFTL) \cite{HazanBook}, and the use of a linear optimization oracle (LOO) to access the feasible set, e.g., \cite{Hazan12, chen2019projection, garber2020improved}. We shall refer to these as RFTL-LOO algorithms. 
Indeed, for many feasible sets of interest and in high-dimensional settings, implementing the LOO can be much more efficient than implementing an orthogonal projection oracle, see many examples in \cite{Jaggi13, Hazan12}. For arbitrary (convex and compact) feasible set and nonsmooth convex losses, the current best regret bound for  both the full-information and bandit settings obtainable by these RFTL-LOO algorithms is $O(T^{3/4})$, using overall $O(T)$ calls to the LOO, due to \cite{Hazan12} and \cite{garber2020improved}. 

However, the RFTL approach for constructing online algorithms has well known inherent limitations. While the regret, as given in \eqref{eq:regret}, can in principle be negative --- due to the ability of the online algorithm to change decisions from iteration to iteration while the benchmark's decision is fixed, it is known that RFTL-type algorithms   \textit{always} suffer non-negative regret \cite{gofer2016lower}. As a consequence, such algorithms are also inherently \textit{non-adaptive} in a sense that we now detail. It is often the case that there is no fixed decision in hindsight that has reasonable performance w.r.t. the entire data (i.e., the sequence of loss functions) and thus, the standard regret measure becomes insufficient. In such cases, \textit{adaptive} performance measures which, on different parts of the data, allow to be competitive against different actions, are much more preferable. Such standard adaptive performance measure introduced in \cite{hazan2009efficient} is called \textit{adaptive regret} and is given by
\begin{align*}
\textrm{Adaptive Regret} = \sup_{[s,e]\subseteq[T]}\left\{{\sum_{t=s}^ef_t(\x_t) - \min_{\x\in\mK}\sum_{t=s}^ef_t(\x)}\right\},
\end{align*}
is the supremum over all standard regrets w.r.t. all sub-intervals of the sequence of loss functions. We refer the interested reader to \cite{hazan2009efficient, daniely2015strongly} for many useful discussions on the adaptive regret and its connection to other notions of adaptivity in the literature. 

Unfortunately, due to their inherent non-negative regret property, RFTL-based algorithms cannot guarantee non-trivial adaptive regret bounds.
Thus, it is natural to ask: 
\begin{center}
\textit{Is it possible to design efficient projection-free algorithms for OCO with non-trivial adaptive regret bounds?}
\end{center}
One attempt towards this goal could be to instantiate the  \textit{strongly adaptive online learner} of \cite{daniely2015strongly} with the non-adaptive state-of-the-art RFTL based algorithm for the full-information setting of \cite{Hazan12}, known as \textit{Online Frank-Wolfe} (OFW), which will result in an adaptive algorithm with $O(T^{3/4})$ adaptive regret.\footnote{In fact, such an algorithm will have for any interval $I$, regret bounded by $O(|I|^{3/4} + |I|^{1/2}\log{}T)$ w.r.t. the interval.} However, this approach is somewhat artificial and will require to run in parallel $O(\log{}T)$ copies of OFW, which will require $\log{T}$-fold memory and calls to the LOO. Moreover, this approach is not applicable to the bandit setting.


Another possibility is to design new projection-free algorithms which are not based on the FTL approach, but instead on the Online Mirror Descent meta-algorithm, and in particular its Euclidean variant --- Online Gradient Descent (OGD) \cite{Zinkevich03}, which  naturally yields an $O(\sqrt{T})$ adaptive regret bound \cite{HazanBook}. While OGD requires to compute on each iteration an orthogonal projection onto the feasible set, a naive approach to making it projection-free using a LOO, is to only approximate the projection on each iteration via the well known Frank-Wolfe method for \textit{offline} constrained minimization of a smooth and convex function, which only uses the LOO \cite{Jaggi13, FrankWolfe}. However, as recently noted in \cite{garber2021efficient}, such an approach strikes an highly suboptimal tradeoff between regret and number of calls to the LOO. Instead, \cite{garber2021efficient} considered using OGD with so-called \textit{infeasible projections}, which on one hand can be computed efficiently with a LOO (at least in terms of the model in \cite{garber2021efficient} which is significantly different than ours), and on the other-hand could be translated into feasible points, without loosing too much in the regret. Our approach in this paper is  inspired by \cite{garber2021efficient}, however, our technique for computing such infeasible projections will be very different (in particular, the setting in \cite{garber2021efficient} is not concerned with the dimension and thus the Ellipsoid method is used, which is not suitable for our setting, due to its polynomial dependence on the dimension).

\paragraph*{Two projection-free oracles:} While our discussion so far has focused on the assumption that the feasible set is accessible through a linear optimization oracle, which is indeed the most popular assumption in the literature on projection-free methods, in this paper we introduce an additional  new natural projection-free setting in which the feasible decision set $\mK$ is given by separation oracle (SO). Given some $\x\in\mK$, the SO either verifies that $\x$ is feasible, in case it indeed holds that $\x\in\mK$ or, returns a hyperplane separating $\x$ from $\mK$, in case $\x\not\in\mK$. 
For instance, a setting in which the SO model arrises naturally is when the feasible set is given by a functional constraint of the form $\mK = \{\x\in\reals^n~|~g(\x) \leq 0\}$, where $g(\cdot)$ is convex. Implementing the SO in this setting simply amounts to calling the first-order oracle of $g(\cdot)$ (i.e., computing $g(\x)$ and some $\g_{\x}\in\partial{}g(\x)$, for a given input point $\x\in\reals^n$). In particular, when $g(\cdot)$ has the following max structure: $g(\x) := \max_{1\leq i \leq m}g_i(\x)$, where $m$ is not very large and $g_1,\dots,g_m$ are convex functions which admit simple structure, implementing the SO can be very efficient, while orthogonal projections can still be prohibitive. One such example is a polytope given by the intersection of $m$ halfspaces for moderately-large $m$. Importantly, the SO model allows to efficiently handle the intersection of several simple convex sets, each given by a SO. Note that in the LOO setting there is no simple approach to implement a LOO for a convex set given as the intersection of several sets, each given by a LOO.


The following example demonstrates the complementing nature of the LOO and SO oracles. Consider the following two, dual to each other, unit balls of matrices which are common in several applications: 
\begin{align*}
\ball_* = \{\X\in\reals^{m\times n} ~|~\Vert{\X}\Vert_* \leq 1\}, ~~ \ball_2 = \{\X\in\reals^{m\times n} ~|~\Vert{\X}\Vert_2 \leq 1\},
\end{align*}
where for a real matrix $\X$ we let $\Vert{\X}\Vert_*$ denote its nuclear/trace norm, i.e., the sum of singular values, and we let $\Vert{\X}\Vert_2$ denote its spectral norm, i.e., its largest singular value. Euclidean projection onto either $\ball_*$ or $\ball_2$ requires in general a full-rank singular value decomposition (SVD), which is computationally prohibitive when both $m,n$ are very large. Linear optimization over $\ball_*$ is quite efficient and only requires a rank-one SVD (leading singular vectors computations) however, linear optimization over $\ball_2$ requires again a full-rank SVD \cite{Jaggi13}. On the other-hand, denoting $g_*(\X) := \Vert{\X}\Vert_*-1, g_2(\X) :=\Vert{\X}\Vert_2-1$, we have that implementing the SO  for $\ball_*$, which requires to compute a subgradient of the nuclear norm, also requires in worst case a full-rank SVD. However, implementing the SO w.r.t. $\ball_2$, requires to compute a subgradient of the spectral norm, which is w.l.o.g. a rank-one matrix (corresponding to a top singular vectors pair of $\X$), and thus requires only a rank-one SVD which is far more efficient. Thus, while a LOO is efficient to implement for $\ball_*$, the SO is efficient to implement for $\ball_2$. 

\paragraph*{Contributions:} \hspace{-4mm} Our main contributions, stated only informally at this stage, and treating all quantities except for $T$ and the dimension $n$ as constants, are as follows (see also a summary in Table \ref{table:Op}).
\begin{enumerate}
\item
Assuming the feasible set is accessible through a LOO, we present an OGD-based algorithm for the full-information setting with adaptive regret of $O(T^{3/4})$ using overall $O(T)$ calls to the LOO. This improves over the previous state-of-the-art (RFTL-based) \textit{not-adaptive} regret bound of $O(T^{3/4})$ due to \cite{Hazan12}. We give a similar algorithm for the bandit setting which guarantees $O(\sqrt{n}T^{3/4})$ adaptive expected regret using $O(T)$ calls to the LOO in expectation, which improves upon the previous best bound of $O(\sqrt{n}T^{3/4})$ due to \cite{garber2020improved} which only applies to the standard regret.
\item
Assuming the feasible set is accessible through a LOO and all loss functions are strongly convex, we show that a projection-free OGD-based algorithm can recover the state-of-the-art $O(T^{2/3})$ (standard) regret bound, up to an additional logarithmic factor,  using $O(T)$ calls to the LOO, which matches (up to a log factor) the RFTL-based method due to \cite{kretzu2021revisiting}.

\item
Assuming the feasible set is accessible through a SO, we present an OGD-based algorithm for the full-information setting with adaptive regret of $O(\sqrt{T})$ using overall $O(T)$ calls to the SO. In the bandit setting, we give a similar algorithm with $O(T^{3/4})$ adaptive expected regret using overall $O(T)$ calls to the SO.
\end{enumerate}

We remark that aside of standard subgradient computations of the loss functions observed, and calls to either the LOO or SO, all of our algorithms require only $O(n)$ space, and $O(nT)$ additional runtime (over all $T$ iterations).

\begin{table}[!ht] \renewcommand{\arraystretch}{1.1}

\begin{center}
  \begin{tabular}{| c | c  | c | c  | c | c  |} \hline
      &  Theorem \ref{thm:LOO-BOGD} & Theorem \ref{thm:LOO-BBGD} & Theorem \ref{thm:SC-OGD-LOO} &  Theorem \ref{thm:OGD-SGO} & Theorem \ref{thm:BGD-SGO} \\ \hline
    Objective & \makecell{adaptive\\regret } & \makecell{adaptive\\ expected regret } & \makecell{regret }  & \makecell{adaptive\\regret }  & \makecell{adaptive\\ expected regret} \\ \hline
    Losses & convex &  convex  &  strongly convex & convex &  convex \\ \hline
    Feedback & full &  bandit  & full  & full  &  bandit  \\ \hline
    Oracle & LOO  & LOO & LOO  & SO & SO \\ \hline
    Regret &  $T^{3/4}$ & $T^{3/4}$ & $ T^{2/3}$ &  $\sqrt{T}$ & $T^{3/4}$ \\ \hline 
  \end{tabular} 
\caption{ Summary of results. For clarity, in the regret bounds we treat all quantities except for $T$ as constants and we omit logarithmic factors.} \label{table:Op}
\end{center}

\end{table}\renewcommand{\arraystretch}{1}

\begin{table}[!ht] \renewcommand{\arraystretch}{1.1}

\begin{center}
  \begin{tabular}{| c | c  | c | c | c  |} \hline
    Feedback  &  Objective  & Oracle & Reference &  Regret \\ \hline
    \multirow{4}{*}{ \makecell{Full \\ Information} }& adaptive regret & projection & \cite{Zinkevich03}  &  $\sqrt{T}$\\ \cline{2-5}
    & adaptive  regret& SO & This work  (Thm. \ref{thm:OGD-SGO}) &  $\sqrt{T}$\\  \cline{2-5}
    & regret & LOO & \cite{Hazan12}  &  $T^{3/4}$\\\cline{2-5}
    & adaptive regret & LOO & This work  (Thm. \ref{thm:LOO-BOGD}) &  $T^{3/4}$\\
     \hline
    \multirow{4}{*}{ \makecell{Bandit  } } & adaptive regret & projection & \cite{Flaxman05} &  $T^{3/4}$ \\ \cline{2-5}
    & adaptive  regret& SO  & This work (Thm. \ref{thm:BGD-SGO}) &   $ T^{3/4}$ \\  \cline{2-5}
     & regret  & LOO & \cite{garber2020improved} &  $T^{3/4}$ \\ \cline{2-5}
    & adaptive regret &  LOO & This work (Thm. \ref{thm:LOO-BBGD}) &   $ T^{3/4}$ \\
    \hline
  \end{tabular}
\caption{Comparison of results to previous works. This is a non-exhaustive list.  Here we only list the most relevant works which are suitable for arbitrary convex and compact sets and convex and nonsmooth losses, make overall $O(T)$ calls to the oracle of the set, and use $O(n)$ memory and $O(nT)$ additional runtime. For clarity, in the regret bounds we treat all quantities except for $T$ as constants.} \label{table:Comp}
\end{center}

\end{table}\renewcommand{\arraystretch}{1}

We acknowledge a parallel work \cite{mhammedi2021efficient}, in which the author proves that given a separation oracle, it is possible to guarantee a $O(\sqrt{T})$ regret bound for general Lipschitz convex losses, and the techniques could  be readily used to also give adaptive regret guarantees in the full information setting (but not in the bandit setting). However, the approach of \cite{mhammedi2021efficient}, which uses substantially different techniques than ours, requires overall $O(T\log{}T)$ calls to the separation oracle to guarantee $O(\sqrt{T})$ regret, while our result only requires $O(T)$ calls in order to achieve this regret bound.

\section{Preliminaries} \label{sec:preli}

\subsection{Additional notation, assumptions and definitions}
Throughout this work we assume without loss of generality that the feasible set $\mK$ contains the origin, i.e., $\mathbf{0}\in\mK$ and we denote by $R> 0$ a radius such that $\mK\subseteq R\ball$, where $\ball$ denotes the unit Euclidean ball centered at the origin. We also denote by $\mathcal{S}$ the unit sphere centered at the origin, and we write $\u \sim \ball$ and $\u \sim \mathcal{S}$ to denote a random vector $\u$ sampled uniformly from $\ball$ and $\mathcal{S}$, respectively.
We assume the loss functions are bounded by $M$ in $\ell_\infty$ norm and are $G_f$-Lipschitz over $R\ball$, that is, for all $t \in [T]$, $\x \in R\ball$ and $\g \in \partial f_t(\x)$, $|f_t(\x)| \leq M$ and $\Vert{\mathbf{g}}\Vert_2\leq G_f$.

In our results for the bandit feedback setting and when assuming the feasible set is accessible through a SO  we shall make the following additional standard assumption.
\begin{assumption} \label{ass:bandit}
    The feasible set fully contains the ball of radius $r$ around $\vz$, for some $r>0$, i.e., $r\ball \subseteq \mK$.
\end{assumption}
For every $\delta\in(0,1)$ we define the $\delta$-squeezed version of $\mK$ as $\mK_{\delta} = (1-\delta)\mK = \{(1-\delta)\x ~|~\x\in\mK\} $. Note that if  Assumption \ref{ass:bandit} holds, then for all $\x\in\mK_{\delta/r}$, it holds that $\x+\delta\ball\subseteq\mK$ (see \cite{HazanBook}).

\subsection{Basic algorithmic tools}



\subsubsection{The Frank-Wolfe algorithm with line search}

The Frank-Wolfe algorithm \cite{FrankWolfe, Jaggi13} is a well known first-order method for minimizing a smooth and convex function over a convex and compact set, accessible through a LOO. In this work we use the Frank-Wolfe with exact line-search variant, see Algorithm \ref{alg:FW-LS}.

\begin{algorithm}[!]
    \KwData{feasible set $\mK$, initial point $\x_0 \in \mK$, objective function $f(\cdot)$.}
    \For{ $i =0, \dots$}{
        $ \mathbf{v}_{i} \in \argmin_{\x \in \mK} \{ \nabla f(\x_i)^{\top} \x \} $ \tcc*{call to LOO of $\mK$}
	    $ \sigma_{i} = \argmin_{\sigma \in [0, 1]}  \{ f(  \x_{i} + \sigma (\mathbf{v}_i - \x_{i}) ) \}$ \\
	    $ \x_{i+1} = \x_i + \sigma_{i} (\mathbf{v}_i - \x_i) $
    }
    \caption{Frank-Wolfe with line-search}\label{alg:FW-LS}
\end{algorithm}

\begin{theorem}{[Primal convergence of FW \cite{Jaggi13}]} \label{thm:Primal_convergence}
    Let $f:\reals^n\rightarrow\reals$ be convex and $\beta$-smooth over a convex and compact set $\mK\subset\reals^n$ with Euclidean diameter $2R$, and denote $\x^* \in \argmin\nolimits_{\x \in \mK}f(\x)$. Algorithm \ref{alg:FW-LS} guarantees that $  \forall i\geq 1: f(\x_i) - f(\x^*) \leq 2\beta{}(2R)^2/(i+2)$.
\end{theorem}

\begin{theorem}{[Dual convergence of FW \cite{Jaggi13}]} \label{thm:Primal-Dual_convergence}
    Under the same assumptions of Theorem \ref{thm:Primal_convergence}, 
    Algorithm \ref{alg:FW-LS} guarantees that  for every number of iterations $K \geq 2$, there exists an iteration $i$, $ K \geq i \geq 2$, such that $\max_{\v \in \mK} (\x_i-\v)^\top \nabla f(\x_i) \leq 6.75\beta{}(2R)^2/(K+2)$.
\end{theorem}

Note that for a convex function $f(\cdot)$ and a feasible point $\x\in\mK$, the dual gap in Theorem \ref{thm:Primal-Dual_convergence} serves as an easy-to-compute certificate for the optimality gap of $\x$ w.r.t. any optimal solution $\x^*\in\argmin_{\y\in\mK}f(\y)$, since from the convexity of $f(\cdot)$ it follows that, $f(\x) - f(\x^*) \leq (\x-\x^*)^{\top}\nabla{}f(\x) \leq \max_{\v \in \mK} (\x-\v)^\top \nabla f(\x)$.

\subsubsection{Online Gradient Descent Without Feasibility}

As discussed, our online algorithms are based on the well known \emph{Online Gradient Descent} method (OGD) \cite{Zinkevich03}, which applies the following updates:
\begin{align*}
    \forall t > 1: ~~~~~ \y_{t+1} \gets \x_t - \eta_t \g_t, ~\g_t\in\partial{}f_t(\x_t), ~~ \x_{t+1} \gets \argmin\nolimits_{\x \in \mK} \Vert \x - \y_{t+1} \Vert^2.
\end{align*}
Where $\{ \eta_t \}_{t=1}^{T}$ are the step-sizes and $\x_1$ is an arbitrary feasible point. However, motivated by \cite{garber2021efficient}, instead of considering exact projections on the feasible set, which may be computationally prohibitive, we consider using only \textit{infeasible projections}, as we now define.  

\begin{definition} \label{def:infeasible_projection}
We say $\Tilde{\y}\in\reals^n$ is an infeasible projection of some $\y\in\reals^n$ onto a convex set $\mK$, if $\forall \z \in \mK$ it holds that $\Vert \Tilde{\y} - \z \Vert^2 \leq \Vert \y - \z \Vert^2 $. We say a function $\mathcal{O}_{IP}(\y,\mK)$ is an infeasible projection oracle for the set $\mK$, if for every input point $\y$, it returns some $\Tilde{\y}\gets\mathcal{O}_{IP}(\y,\mK)$ which is an infeasible projection of $\y$ onto $\mK$.
\end{definition}

This definition gives rise to the online gradient descent without feasibility algorithm --- Algorithm \ref{alg:OGD-WF}, and its corresponding regret bounds captured in Lemma \ref{lemma:OGD-WF}. While this algorithm will play a central role in our projection-free online algorithms, clearly, another central piece, which we will detail later on, will be  to transform such infeasible projections into feasible points without loosing too much in the regret bound. 

\begin{algorithm}[!ht]
    \KwData{horizon $T$, feasible set $\mK$, step-sizes $\{ \eta_t \}_{t=1}^{T}$, infeasible projection oracle $\mathcal{O}_{IP} (\mK, \cdot)$}
    $\Tilde{\y}_1 \gets $ arbitrary point in $\mK$\\
    \For{$~ t = 1,\ldots,T ~$}{
        Play $\Tilde{\y}_{t} $, observe $f_{t}(\Tilde{\y}_{t})$, and set $\nabla_t \in  \partial f_t(\Tilde{\y}_{t})$\\
        Update $\y_{t+1} = \Tilde{\y}_{t} - \eta_t \nabla_t$, and set $\Tilde{\y}_{t+1} \gets \mathcal{O}_{IP} (\mK, \y_{t+1})$
 }
\caption{Online Gradient Descent Without Feasibility}\label{alg:OGD-WF}
\end{algorithm}

\begin{lemma}\label{lemma:OGD-WF}
    Let $\mathcal{O}_{IP}$ an infeasible projection oracle (Definition \ref{def:infeasible_projection}).
    \begin{enumerate}
        \item Suppose all loss functions are convex. Fix some $\eta >0$ and let $\eta_t = \eta$ for all $t\geq 1$.  Algorithm \ref{alg:OGD-WF} guarantees that the adaptive regret is upper-bounded as follows:
    \begin{align*}
         \forall I=[s,e]\subseteq[T]: \quad \sum_{t=s}^{e} f_t(\Tilde{\y}_{t}) - \min_{\x_I \in \mK} \sum_{t=s}^{e} f_t(\x_I)  \leq \frac{ \left\Vert   \Tilde{\y}_s - \x_I \right\Vert^2 }{2\eta} + \frac{\eta  }{2}  \sum_{s=1}^{e} \Vert \nabla_t \Vert^2.
    \end{align*}
        \item Suppose all loss functions are $\alpha$-strongly convex for some $\alpha > 0$. Let $\eta_t = \frac{1}{\alpha{}t}$ for all $t\geq 1$. Algorithm \ref{alg:OGD-WF} guarantees that the (static) regret is upper-bounded as follows:
    \begin{align*}
        \sum_{t=1}^{T} f_t(\Tilde{\y}_{t}) - \min_{\x \in \mK} \sum_{t=1}^{T} f_t(\x) \leq & \sum_{t =1}^{T}   \Vert \nabla_t \Vert^2 / 2 \alpha t .
    \end{align*}
    \end{enumerate}
\end{lemma}

The proof which follows from standard analysis of OGD (see for instance \cite{HazanBook}) is given in the appendix  for completeness.

\subsubsection{Infeasible projections via separating hyperplanes}
Continuing the discussion on infeasible projections, our approach for transforming such infeasible projections into feasible points without sacrificing the regret bounds too much, will be to design infeasible projection oracles that always return points that are sufficiently close to the feasible set. The following simple lemma will be instrumental to all of our constructions of such oracles, and shows how using a separating hyperplane we can ``pull'' an infeasible point closer to the feasible set.
\begin{lemma}\label{lemma:update_step_with_hp}
   Let $\mK\subset\reals^n$ be convex and compact, let $\y$ be infeasible w.r.t. $\mK$, i.e., $\y\notin \mK$, and let $\g\in\reals^n$ be a separating hyperplane such that for all $\z\in\mK$: $(\y-\z)^{\top} \g \geq Q$, for some $Q \geq 0$. Consider the point $\tilde{\y} = \y - \gamma\g$, for $\gamma = Q/C^2$, where $C \geq \Vert{\g}\Vert$. It holds that
\begin{align*}
   \forall \z\in\mK: \Vert \Tilde{\y} -\z \Vert^2 \leq \left\Vert \y -\z  \right\Vert^2 - (Q/C)^2.
\end{align*}
\end{lemma}

\begin{proof}
Fix some $\z\in\mK$. It holds that
\begin{align*}
    \Vert \Tilde{\y} -\z \Vert^2 = \left\Vert \y -\z - \gamma  \g \right\Vert^2   \leq \left\Vert \y -\z  \right\Vert^2 - 2 \gamma (\y -\z )^\top \g + \gamma^2 C^2.
\end{align*}
Since $\left( \y - \z \right)^\top \g \geq Q$, we indeed obtain 
\begin{align*}
    \Vert \Tilde{\y} -\z \Vert^2 \leq \left\Vert \y -\z  \right\Vert^2 - 2 \gamma Q + \gamma^2 C^2 \leq \Vert \y -\z \Vert^2 - Q^2/C^2,
\end{align*}
where the last inequality follows from plugging-in the value of $\gamma$.
\end{proof}

\subsubsection{Smoothed loss functions for bandit optimization}\label{sec:smooth}
A standard component of bandit algorithms \cite{Flaxman05, pmlr-v80-chen18c, garber2020improved, kretzu2021revisiting},  is the use of smoothed versions of the loss functions and their unbiased estimators. We define the $\delta$-smoothing of a loss function $f$ by { $\widehat{f}_{\delta} (\x) = \E_{\u \sim \ball} \left[ f(\x + \delta \u) \right]$ }.
We now cite several standard useful lemmas regarding such smoothed functions.
\begin{lemma}[Lemma 2.1 in \cite{HazanBook}] \label{lemma:hazan_smooth}
    Let $f: \reals^n \xrightarrow{} \reals$ be convex and $G_f$-Lipschitz over a convex and compact set $\mK\subset\reals^n$. Then $\widehat{f}_{\delta}$ is convex and $G_f$-Lipschitz over $\mK_{\delta}$, and $\forall \x \in \mK_{\delta}$ it holds that $|\widehat{f}_{\delta} (\x) - f(\x)| \leq \delta G_f$. 
\end{lemma}
\begin{lemma}[Lemma 6.5 in \cite{HazanBook}] \label{lemma:hazan_gradient}
	$\widehat{f}_{\delta}(\x)$ is differentiable and $\nabla \widehat{f}_{\delta}(\x) = \E_{\u \sim \mS^n} \left[ \frac{n}{\delta} f(\x + \delta \u)\u \right]$, where $\u$ is sampled uniformly from $S^n$.
\end{lemma}
\begin{lemma} [see \cite{Bertsekas73}]
\label{lemma:bertsekas_grdient}
    Let $f: \reals^n \xrightarrow{} \reals$ be convex and suppose that all subgradients of $f$ are upper-bounded by $G_f$ in $\ell_2$-norm over a convex and compact set $\mK\subset\reals^n$. Then, for any $\x\in\mK_{\delta}$ it holds that $\Vert{\nabla{}\widehat{f}_{\delta}(\x)}\Vert \leq G_f$.
\end{lemma}

\section{Projection-free Algorithms via a Linear Optimization Oracle}

In this section we present and analyze our LOO-based algorithms. 

\subsection{LLO-based computation of (close) infeasible projections}
The main step towards obtaining our novel algorithms will be to construct an efficient LOO-based infeasible projection oracle (Definition \ref{def:infeasible_projection}). A first step towards this goal will be to show how an LOO could be efficiently used to construct separating hyperplanes for the feasible set $\mK$. This will be achieved via the Frank-Wolfe algorithm, when applied to computing the Euclidean projection of the given point onto $\mK$, i.e., to solve $\min_{\x\in\mK}\Vert{\x-\y}\Vert^2$, where $\y$ is the point which should be separated from $\mK$. See Algorithm \ref{alg:SH-FW} and the corresponding Lemma \ref{lemma:SH-FW}.
\begin{algorithm}[!]
  \KwData{feasible set $\mK$, error tolerance $\epsilon$, initial vector $\x_1 \in \mK$, target vector $\y$.}
  \For{ $i =1, \dots$}{
        $ \mathbf{v}_{i} \in \argmin\limits_{\x \in \mK} \{ (\x_{i} - \y)^{\top} \x \} $\tcc*{call to LOO of $\mK$}
        \uIf{$( \x_i - \y )^\top (\x_i -\v_i) \leq \epsilon$ or $\Vert \x_{i} - \y \Vert^2 \leq 3\epsilon$}{
	        \textbf{return} $\Tilde{\x} \gets \x_{i}$
        }
	    $ \sigma_{i} = \argmin\limits_{\sigma \in [0, 1]}  \{ \Vert \y - \x_{i} - \sigma (\mathbf{v}_i - \x_{i})) \Vert^2 \}$\\
	$ \x_{i+1} = \x_i + \sigma_{i} (\mathbf{v}_i - \x_i) $\\
    }
  \caption{Separating hyperplane via Frank-Wolfe}\label{alg:SH-FW}
\end{algorithm}

\begin{lemma} \label{lemma:SH-FW} 
Fix $\epsilon > 0 $. Algorithm \ref{alg:SH-FW} terminates after at most $\left\lceil \left(27 R^2 / \epsilon \right) -2 \right\rceil$ iterations, and returns a point $\Tilde{\x} \in \mK$  satisfying:
\begin{enumerate}
\item
$\Vert \Tilde{\x} - \y \Vert^2 \leq \Vert \x_1 - \y \Vert ^2$.
\item
At least one of the following holds:  $\Vert \Tilde{\x} - \y \Vert^2 \leq 3\epsilon$ or  $\forall \z \in \mK:  (\y - \z)^\top (\y - \Tilde{\x}) > 2\epsilon$.
 \item
If $\dist^2 (\y, \mK)  < \epsilon$ then $\Vert \Tilde{\x} - \y \Vert^2 \leq 3\epsilon$.
\end{enumerate}
\end{lemma}
\begin{proof}
Since Algorithm \ref{alg:SH-FW} is simply the Frank-Wolfe method with line-search (Algorithm \ref{alg:FW-LS}) when applied to the function $f(\x) = \frac{1}{2}\Vert{\x-\y}\Vert^2$, which is $1$-smooth and with gradient vector $\nabla{}f(\x) = \x-\y$, the upper-bound on the number of iterations follows directly from Theorem \ref{thm:Primal-Dual_convergence}, which guarantees that the stopping condition of the algorithm will be met within the prescribed number of iterations. 

Similarly, Item 1 in the theorem follows directly since the line-search guarantees that the function value $f(\x_i) = \frac{1}{2}\Vert{\x_i-\y}\Vert^2$ does not increase when moving from iterate $\x_i$ to $\x_{i+1}$.

Item 2 follows from the stopping condition of the algorithm and by noting that in case for some iteration $i$ it  holds both that $(\x_i - \y )^\top (\x_i -\v_i) \leq \epsilon$ and $\Vert{\x_i-\y}\Vert^2 > 3\epsilon$ (in which case the algorithm will return $\tilde{\x} = \x_i$), then for all $ \z\in\mK$ it holds
\begin{align*}
  \left( \z - \y \right)^\top \left( \x_{i} - \y \right)  = \left( \z - \x_{i} \right)^\top \left( \x_{i} - \y \right) + \Vert \x_{i} - \y \Vert^2  & > \left( \v_{i} - \x_{i} \right)^\top \left( \x_{i} - \y \right) + 3\epsilon > 2\epsilon,
\end{align*}
where the first inequality is due to the definition of $\v_i$. Finally, to prove Item 3, denote $\x^* = \argmin_{\x\in\mK}\Vert{\x-\y}\Vert^2$. Suppose by contradiction that $\dist^2(\y,\mK) = \Vert{\x^*-\y}\Vert^2 < \epsilon$ and that $\Vert{\tilde{\x}-\y}\Vert^2 > 3\epsilon$. Denote the function $f(\x) = \frac{1}{2}\Vert{\x-\y}\Vert^2$ and its gradient vector $\nabla{}f(\x) = \x-\y$. According to the assumption and by the stopping condition of the algorithm, on the last iteration executed $i$ it must hold that $(\tilde{\x}-\y)^{\top}(\tilde{\x}-\v_i)  = \max_{\v\in\mK}(\tilde{\x}-\v)^{\top}\nabla{}f(\tilde{\x}) \leq \epsilon$, which means that
\begin{align*}
	\Vert{\tilde{\x}-\y}\Vert^2 - \dist^2(\y,\mK) = 2f(\tilde{\x}) - 2f(\x^*) \leq 2(\tilde{\x} - \x^*)^{\top}\nabla{}f(\tilde{\x}) \leq  2\max_{\v\in\mK}(\tilde{\x}-\v)^{\top}\nabla{}f(\tilde{\x}) \leq 2\epsilon,
\end{align*}
where the first inequality is due to the gradient inequality and the convexity of $f(\cdot)$. Thus, we have that $\Vert{\tilde{\x}-\y}\Vert^2 \leq 2\epsilon + \dist^2(\y,\mK) \leq 3\epsilon$, which contradicts the assumption that $\Vert{\tilde{\x}-\y}\Vert^2 > 3\epsilon$.
\end{proof}

We can now use Algorithm \ref{alg:SH-FW} as a subroutine in an iterative algorithm which takes as input some infeasible point $\y\notin\mK$, and returns an infeasible projection of it w.r.t. the feasible set $\mK$ that is also guaranteed to be at a bounded distance for $\mK$. In a nutshell, as long as the infeasible point is too far from the set, Algorithm \ref{alg:CIP-FW} iteratively calls Algorithm \ref{alg:SH-FW} to obtain a separating hyperplane which is then used to ``pull'' the point closer to the set while maintaining the infeasible projection property.


\begin{algorithm}[!]
  \KwData{feasible set $\mK$, feasible point $\x_{0} \in \mK$, initial point $\y_{0}$, error tolerance $\epsilon$, step size $\gamma$}
  $\y_{1} \gets \y_{0} / \max \{ 1 , \Vert \y \Vert / R \} $ \tcc*{$\y_{1}$ is projection of $\y_{0}$ over $R\ball$}
  \If{$\Vert \x_{0} - \y_{0} \Vert^2 \leq 3\epsilon$}{
        \textbf{Return} $\x \gets \x_{0}$, $\y \gets \y_{1}$
    }
   \For{$i=1 \dots$}{
    $\x_{i} \gets$ Output of Alg. \ref{alg:SH-FW} with set $\mK$, feasible point $\x_{i-1}$, initial vector $\y_{i}$, and tolerance $\epsilon$.\\
    \eIf{$\Vert \x_{i} - \y_{i} \Vert^2 > 3\epsilon$}{
        $\y_{i+1} = \y_{i} - \gamma \left( \y_{i} - \x_{i} \right)$ \tcc*{$\left( \y_{i} - \x_{i} \right)$ separates $\y_i$ from $\mK$} 
    }
    {
    \textbf{Return} $\x \gets \x_{i}$, $\y \gets \y_{i}$
    }
  }
  \caption{Close infeasible projection via a linear optimization oracle}\label{alg:CIP-FW}
\end{algorithm}

\begin{lemma} \label{lemma:CIP-FW}
Fix $\epsilon > 0 $. Setting $\gamma= \frac{2\epsilon}{\Vert \x_{0} - \y_{0} \Vert^2}$, Algorithm \ref{alg:CIP-FW} stops after at most $ \max \Big\{\frac{\Vert \x_{0} - \y_{0} \Vert^2 \left(\Vert \x_{0} - \y_{0} \Vert^2 - \epsilon \right)}{4\epsilon^2}+1, 1 \Big\}$ iterations, and returns $(\x,\y) \in \mK\times R\ball$ scuh that 
\begin{align*}
    \forall \z \in \mK : ~ \Vert \y - \z \Vert^2 \leq  \Vert \y_{0} - \z \Vert^2 ~~~~ \text{and} ~~~~~   \Vert \x - \y \Vert^2 \leq 3\epsilon.
\end{align*}
Furthermore, if the for loop has completed overall $k$ iterations, then the point $\y$ satisfies  
\begin{align*}
    \dist^2 (\y, \mK) \leq \min \Big{\{} R ,  \dist^2 (\y_0, \mK) - (k-1) 4\epsilon^2 / \Vert \x_{0} - \y_{0} \Vert^2 \Big{\}}.
\end{align*}
\end{lemma}

Before proving Lemma  \ref{lemma:CIP-FW} we require and additional auxiliary lemma.

\begin{lemma}\label{lemma:CIP-FW_aux}
Consider Algorithm \ref{alg:CIP-FW} and fix some $\epsilon$ such that $0 < 3\epsilon < \Vert \x_{0} - \y_{0} \Vert^2$. Setting $\gamma= \frac{2\epsilon}{\Vert \x_{0} - \y_{0} \Vert^2}$, we have that on every iteration $i$ of Algorithm \ref{alg:CIP-FW} it holds that $\Vert \x_{i} - \y_{i} \Vert \leq \Vert \x_{0} - \y_{0} \Vert$.
\end{lemma}

\begin{proof}
Using the update step of the algorithm, for every iteration $i > 1$ we have that $\y_{i} = \y_{i-1} - \gamma \left( \y_{i-1} - \x_{i-1} \right)$, and thus, for every $i>1 $ we have that
\begin{align*}
     \Vert \x_{i-1} - \y_{i} \Vert  = \left\Vert \x_{i-1} - \y_{i-1} + \gamma \left( \y_{i-1} - \x_{i-1}  \right) \right\Vert  = \left( 1 - \gamma \right) \left\Vert \x_{i-1} - \y_{i-1} \right\Vert \leq  \Vert \x_{i-1} - \y_{i-1} \Vert, 
\end{align*}
where the last inequality holds since our choice of $\gamma$ satisfies $\gamma\in[0,1)$.

From Lemma \ref{lemma:SH-FW} we also have that for all $i\geq 1$, $\x_i$ satisfies $\Vert \x_i - \y_i \Vert \leq \Vert \x_{i-1} - \y_i \Vert $. Combining this with the inequality above gives
\begin{align*}
    \Vert \x_{i} - \y_{i} \Vert & \leq  \Vert \x_{i-1} - \y_{i-1} \Vert \leq \dots\leq \Vert \x_{1} - \y_{1} \Vert \leq \Vert \x_0- \y_1 \Vert \leq \Vert \x_0 - \y_0 \Vert, 
\end{align*}
where the last inequality follows since $ \x_0 \in \mK$ and $\y_{1} \gets \y_{0} / \max \{ 1 , \Vert \y_{0} \Vert / R \} $, i.e. $\y_{1}$ is the projection of $\y_{0}$ over the set $R\ball$ ($\mK \subseteq R\ball$), and thus $\Vert \x_0 - \y_1 \Vert \leq \Vert \x_0 - \y_0 \Vert $.
\end{proof}

\begin{proof} [Proof of Lemma \ref{lemma:CIP-FW}] 
First, we note that since $\y_1$ is the projection of $\y_0$ onto $R\ball $ and $\mK  \subseteq R\ball$, it holds that $\forall  \z \in \mK:  \Vert \y_{1} - \z \Vert^2 \leq  \Vert \y_0 - \z \Vert^2$. When $\Vert \x_{0} - \y_{0} \Vert^2 \leq 3\epsilon$ or $\Vert \x_{1} - \y_{1} \Vert^2 \leq 3\epsilon$ the lemma holds trivially. 

For the remaining of the proof we shall assume that $\Vert \x_{1} - \y_{1} \Vert^2 > 3\epsilon$. Let us denote by $k>1$ the overall number of iterations of Algorithm \ref{alg:CIP-FW}, i.e. $\Vert \y_{k} - \x_{k} \Vert^2 \leq 3 \epsilon$ and $\Vert \y_{i} - \x_{i} \Vert^2 > 3 \epsilon$ for all $i < k $. Using Lemma \ref{lemma:SH-FW}, we have that for all $i < k$ it holds that $\left( \y_{i} - \z \right)^\top \left( \y_{i} - \x_{i} \right) \geq 2\epsilon$ for every $\z \in \mK$. Using Lemma \ref{lemma:CIP-FW_aux} we also have that $\Vert \y_{i} - \x_{i} \Vert \leq \Vert \y_{0} - \x_{0} \Vert$ for all $i < k$. Thus, using Lemma \ref{lemma:update_step_with_hp} with $\g = \left( \y_{i} - \x_{i} \right) , C= \Vert \y_{0} - \x_{0} \Vert, \text{and }Q=2\epsilon $ , we have that for every $1 \leq i < k$,
\begin{align}
    \forall\z\in\mK:\quad \Vert \y_{i+1} -\z \Vert^2 \leq \Vert \y_{i} -\z \Vert^2 - 4 \epsilon^2 / \Vert \y_{0} - \x_{0} \Vert^2. \label{eq:update_hyperplane_linear_optimization_oracle}
\end{align}
This already guarantees that indeed forall $\z\in\mK$, the returned point $\y$ satisfies: $\Vert \y - \z \Vert^2 \leq  \Vert \y_{1} - \z \Vert^2 \leq  \Vert \y_{0} - \z \Vert^2$. Since $\vz \in \mK$, it in particular follows that $\Vert \y \Vert \leq \Vert \y_1 \Vert \leq R $, i.e. $\y \in R\ball $.  Note also that $\x\in\mK$ since it is the output of Algorithm \ref{alg:SH-FW}.

Now we continue to upper-bound the number of iterations until Algorithm \ref{alg:CIP-FW} stops and $\dist^2(\y,\mK)$. Denote $\x_{i}^* = \argmin_{\x \in \mK} \Vert \y_{i} - \x \Vert^2$ for every iteration $i < k$. Using Eq. \eqref{eq:update_hyperplane_linear_optimization_oracle}, for every iteration $i < k$ it holds that 
\begin{align*}
    \dist^2 (\y_{i+1}, \mK)  & = \Vert  \y_{i+1} - \x_{i+1}^* \Vert^2 \leq \Vert \y_{i+1} - \x_{i}^* \Vert^2   \\
    & \leq \Vert \y_{i} - \x_{i}^* \Vert^2 -  4 \epsilon^2 / \Vert \y_{0} - \x_{0} \Vert^2 = \dist^2 (\y_{i}, \mK)  -  4\epsilon^2 / \Vert \y_{0} - \x_{0} \Vert^2.
\end{align*}
Unrolling the recursion and using $\dist^2 (\y_{1}, \mK) \leq \dist^2 (\y_{0}, \mK) $, we have
\begin{align*}
    \dist^2 (\y_{i+1}, \mK)  & \leq \dist^2 (\y_{1}, \mK) - i 4\epsilon^2 / \Vert \y_{0} - \x_{0} \Vert^2 \leq \dist^2 (\y_{0}, \mK)  - i 4\epsilon^2 / \Vert \y_{0} - \x_{0} \Vert^2\\
    & \leq \Vert \y_{0} - \x_{0} \Vert^2 - i 4\epsilon^2 / \Vert \y_{0} - \x_{0} \Vert^2,
\end{align*}
Then, after at most $k-1 = \left(\Vert \y_{0} - \x_{0} \Vert^2 (\Vert \y_{0} - \x_{0} \Vert^2 - \epsilon)\right)/ 4\epsilon^2$ iterations, we obtain $\dist^2 (\y_{k}, \mK) \leq \epsilon$, which by using Lemma \ref{lemma:SH-FW}, implies that the next iteration will be the last one, and the returned points $\x,\y$ will indeed satisfy$\Vert \x - \y \Vert^2 \leq 3\epsilon$, as required.
\end{proof}

\subsection{LOO-based algorithms for the full-information setting}

We are now ready to fully detail our algorithm for the full information setting using a LOO, Algorithm \ref{alg:LOO-BOGD}, and analyze its regret and oracle complexity. The algorithm combines the OGD without feasibility algorithm, Algorithm \ref{alg:OGD-WF}, and the LOO-based infeasible projection oracle given in Algorithm \ref{alg:CIP-FW}. Since  each invokation of Algorithm \ref{alg:CIP-FW} may call (through Algorithm \ref{alg:SH-FW} the LOO several times, Algorithm \ref{alg:LOO-BOGD} considers the iterations in blocks of $K$ disjoint iterations ($K$ is a parameter to be determined in the analysis), and uses the same prediction for the entire block. Thus, a single call to the infeasible projection oracle, Algorithm \ref{alg:CIP-FW}, is made on each block. Finally, we note that for more practical considerations, the update to predictions of the algorithm is delayed in such a way that, at the end of each block $m$, the algorithm does not need to wait until the prediction for the next block $m+1$ will be computed but, it is already computed during the course of block $m$.

\begin{algorithm}[!ht]
\KwData{horizon $T$, feasible set $\mK$, block size $K$, update steps $\{ \eta_m \}_{m=1}^{T/K}$, error tolerances $\{ \epsilon_m \}_{m=1}^{T/K}$.}
$\x_0,\x_1 \gets $ arbitrary points in $\mK$.\\
$\Tilde{\y}_0 \gets \x_0, \y_1 \gets \Tilde{\y}_0, \Tilde{\y}_1 \gets \x_1$.\\
\For{$~ t = 1,\ldots,K ~$}{
    Play $\x_{0} $ and observe $f_{t}(\x_{0})$\\
    Set $\tilde{\nabla}_t \in \partial f_t(\tilde{\y}_{0})$ and update $\y_{t+1} = \y_{t} - \eta_1 \tilde{\nabla}_t$ 
}
\For{$~ m = 2,\ldots,\frac{T}{K} ~$}{
    Let $(\x_{m},\tilde{\y}_{m})\in\mK\times{}R\ball$ be the output of Algorithm \ref{alg:CIP-FW} when called with set $\mK$, feasible point $\x_{m-2}$, initial point $\y_{(m-1)K+1}$, and tolerance $\epsilon_m$
    (execute \textbf{in parallel} to the following \textbf{for} loop over $s$)\\
    Set $\y_{(m-1)K+1} = \Tilde{\y}_{m-1}$\\
    \For{$~ s = 1,\ldots,K ~$}{
        Play $\x_{m-1} $ and observe $f_{t}(\x_{m-1})$ \tcc*{$t = (m-1)K+s$}
        Set $\tilde{\nabla}_t \in \partial  f_t(\tilde{\y}_{m-1})$ and update $\y_{t+1} = \y_{t} - \eta_m  \tilde{\nabla}_t$ 
        }
    \textbf{Note:} $\y_{mK+1} = \Tilde{\y}_{m-1} - \eta_m \sum_{t=(m-1)K+1}^{mK} \tilde{\nabla}_t$. 
}
\caption{Blocked Online Gradient Descent with LOO (LOO-BOGD) }\label{alg:LOO-BOGD}
\end{algorithm}

\begin{theorem}\label{thm:LOO-BOGD}   
Setting  $\eta_m = \eta = (R / G_f)T^{-\frac{3}{4}}$, $\epsilon_m = \epsilon= 60 R^2 T^{-\frac{1}{2}}$ for all $m\geq 1$, and  $ K = 5T^{\frac{1}{2}}$
in Algorithm \ref{alg:LOO-BOGD}, guarantees that the adaptive regret is upper bounded by
\begin{align*}
    \sup_{ I = [s,e] \subseteq [T]} \left\{ \sum_{t=s}^{e} f_t(\x_{t}) - \min_{\x_I \in \mK} \sum_{t=s}^{e} f_t(\x_I) \right\} \leq 20 G_f R T^{\frac{1}{2}} +  20 G_f R  T^{\frac{3}{4}},
\end{align*}
and that the overall number of calls to the LOO is upper bounded by 
\begin{align*}
    N_{calls} & \leq   T.
\end{align*}
\end{theorem}

Before proving the theorem we need an additional lemma.

\begin{lemma}\label{lemma:LOO-BOGD-WF}
Let $\{ \Tilde{\y}_m \}_{m=2}^{\frac{T}{K}-1}\subset  R\ball$ be as in Algorithm \ref{alg:LOO-BOGD} when ran with some block size $K$, for some positive integer $K$, and step-sizes $\eta_m = \eta >0$ for all $m\geq 1$ . It holds that
\begin{align*}
    & \sup\limits_{ I = [s,e]\subseteq [T] } \bigg{\{} \sum_{t=s}^{e} f_t(\Tilde{\y}_{t}) - \min\limits_{\x_I \in \mK} \sum_{t=s}^{e} f_t(\x_I) \bigg{\}}  \leq \frac{ \eta  }{2} K G_f^2 T +   4 R K G_f  + \frac{ 4R^2}{\eta }.
\end{align*}
\end{lemma}

\begin{proof}
Denote $\mT_m = \{ (m-1)K+1, \dots, mK \}$ for every $m \in [T/K]$. Since for every $m \in [T/K]$, $\Tilde{\y}_{m+1}$ is the output of Algorithm \ref{alg:CIP-FW} when called with the input $\y_{mK+1}$, we have from Lemma \ref{lemma:CIP-FW} that  $\forall \x \in \mK: ~ \Vert \Tilde{\y}_{m+1} - \x \Vert^2 \leq \Vert \y_{mK+1} - \x \Vert^2 $. Note also that  $\y_{mK+1} = \Tilde{\y}_{m-1} -  \eta \sum_{t \in \mT_{m}}  \Tilde{\nabla}_t$, where $\Tilde{\nabla}_t \in \partial f_t(\Tilde{\y}_{m-1})$. Thus,  we have that 
\begin{align*}
    \forall\x\in\mK: \quad \Vert \Tilde{\y}_{m+1}  - \x \Vert^2 & \leq \Vert  \y_{mK+1} - \x \Vert^2 = \left\Vert  \Tilde{\y}_{m-1} -  \eta \sum\limits_{t \in \mT_{m}} \Tilde{\nabla}_t -\x \right\Vert^2\\
    & \leq \left\Vert \Tilde{\y}_{m-1} - \x \right\Vert^2 + \eta^2 K^2 G_f^2 - 2 \eta \sum\limits_{t \in \mT_{m}} \Tilde{\nabla}_t^\top (\Tilde{\y}_{m-1} - \x),
\end{align*}
where in the last inequality we have used the assumption that for all $t\in [T]$ and $\x \in R\ball$ it holds $ \Vert \nabla f_t(\x) \Vert  \leq G_f $.

Rearranging, we have for every block $m$ that 
\begin{align}
    \sum\limits_{t \in \mT_{m}}  \Tilde{\nabla}_t^\top & (\Tilde{\y}_{m-1} - \x)  \leq  \frac{ \left\Vert   \Tilde{\y}_{m-1} - \x \right\Vert^2}{2\eta} - \frac{ \Vert \Tilde{\y}_{m+1} - \x \Vert^2}{2\eta} + \frac{\eta  }{2} K^2 G_f^2. \label{eq:BOGD-WF_one_iteration_bound}
\end{align}
Fix some interval $[s,e], 1\leq s \leq e\leq T$. We define two scalars $m_s$ and $m_e$ which are set to the smallest block index and the largest block index that are fully contained in the interval $[s,e]$, respectively.  Recall that for a certain block $m$, all iterations $t \in \mT_{m}$ share the same prediction $\Tilde{\y}_{m-1}$.
Thus,  for every $\x \in \mK$ we have that
\begin{align*}
    \sum\limits_{t =s}^{e}  \Tilde{\nabla}_t^\top  \left( \Tilde{\y}_{m(t)-1} - \x \right)  \leq  &  \sum\limits_{t =s}^{m_{s-1}K}  \Tilde{\nabla}_t^\top ( \Tilde{\y}_{m_{s-2}} - \x) + \sum\limits_{ m = m_s}^{m_e} \sum\limits_{ t \in \mT_m }   \Tilde{\nabla}_t^\top ( \Tilde{\y}_{m-1} - \x)   \\
    &  + \sum\limits_{t =m_e K +1}^{e}\Tilde{\nabla}_t^\top ( \Tilde{\y}_{m_{e}} - \x).
\end{align*}
Using the Cauchy-Schwarz inequality, recalling that $\Tilde{\y}_{m} \in R \ball$ for all $m$, and $\Vert{\nabla{}f_t(\z)}\Vert \leq G_f$ for all $t\geq 1$ and $\z\in\R\ball$, we have that $\Tilde{\nabla}_t^\top ( \Tilde{\y}_{m(t)} - \x) \leq 2 G_f R $ for every $t\geq 1$ and $\x \in \mK$. Using Eq.\eqref{eq:BOGD-WF_one_iteration_bound}, and this last observation, we have that for every $\x\in\mK$,
\begin{align*}
    \sum\limits_{t =s}^{e} \Tilde{\nabla}_t^\top  \left( \Tilde{\y}_{m(t)-1} - \x \right) \leq & \sum\limits_{m=m_s}^{m_e} \left(  \frac{ \left\Vert   \Tilde{\y}_{m-1} - \x \right\Vert^2}{2\eta} - \frac{ \Vert \Tilde{\y}_{m+1} - \x \Vert^2}{2\eta} + \frac{K^2 \eta G_f^2 }{2}\right) + 4 K G_f R.
\end{align*}
Since for every $t \in [T]$, $ f_t (\cdot) $ is convex in $R \ball$, we have that
\begin{align*}
\forall\x\in\mK:~  \sum\limits_{t =s}^{e} f_t \left(\Tilde{\y}_{m(t)-1}\right) -  f_t(\x) \leq & \frac{ 4R^2}{\eta} + \frac{ \eta  }{2} K G_f^2 T +  4 K G_f R,
\end{align*}
and thus the lemma follows.
\end{proof}

\begin{proof}[Proof of Theorem \ref{thm:LOO-BOGD}]
 Denote $m(t) = \left\lceil \frac{t}{K} \right\rceil$ and $\nabla_t \in \partial f_t(\x_{m(t)-1})$ for all $t\in[T]$. Fix some interval $[s,e], 1\leq s\leq e\leq T$, and fix some minimizer $\x_I^* \in \argmin_{\x \in \mK} \sum_{t=s}^{e} f_t(\x)$. From the convexity of $ f_t(\cdot) $ for every $t \in [T]$, we have that for every $\x \in R\ball$ it holds that
\begin{align}
    \sum_{t=s}^{e} f_t  \left( \x_{m(t)-1} \right) - f_t(\x  ) & = \sum_{t=s}^{e} f_t \left( \x_{m(t)-1} \right) - f_t \left( \Tilde{\y}_{m(t)-1} \right) + f_t \left( \Tilde{\y}_{m(t)-1} \right) - f_t(\x) \nonumber \\
    & \leq  \sum_{t=s}^{e} \nabla_t^\top \left( \x_{m(t)-1} - \Tilde{\y}_{m(t)-1} \right) + \sum_{t=s}^{e} f_t \left( \Tilde{\y}_{m(t)-1} \right) - f_t(\x). \label{eq:LOO-BOGD_full_regret_arbitrary}
\end{align}
Using Lemma \ref{lemma:CIP-FW}, for every block $m$, Algorithm \ref{alg:CIP-FW} returns points $(\x_{m},\Tilde{\y}_{m})\in\mK\times{}R\ball$ such that $ \Vert \x_{m} - \Tilde{\y}_{m} \Vert^2 \leq 3 \epsilon$ (recall that $\epsilon_m=\epsilon$ for every $m\in[T/K]$). Since for every $t\geq 1$, $f_t(\cdot) $ is $G_f-$Lipschitz over $R\ball$, from both observations and using the Cauchy-Schwarz inequality, for every $t\in[T]$ we have  that, 
\begin{align}
     \nabla_t^\top \left( \x_{m(t)-1} - \Tilde{\y}_{m(t)-1} \right) & \leq  G_f \left\Vert \x_{m(t)-1} - \Tilde{\y}_{m(t)-1} \right\Vert   \leq   G_f \sqrt{3\epsilon}. \label{eq:LOO-BOGD_x_y_regret}
\end{align}
Since Eq.\eqref{eq:LOO-BOGD_full_regret_arbitrary} holds for any interval $[s,e]$, using Eq.\eqref{eq:LOO-BOGD_x_y_regret} and the fact that $\sup_{x} \{ f_1(\x) + f_2(\x) \} \leq \sup_{x} \{ f_1(\x) \} + \sup_{x} \{ f_2(\x) \}$, we have that
\begin{align*}
    \sup\limits_{I=[s,e]\subseteq[T]} \Bigg{\{} \sum_{t=s}^{e} f_t\left(\x_{m(t)-1}\right) - \sum_{t=s}^{e} f_t(\x_I^*) \Bigg{\}}  \leq & \sup\limits_{I=[s,e]\subseteq[T]} \bigg{\{} \sum_{t=s}^{e} f_t\left(\Tilde{\y}_{m(t)-1}\right) - \sum_{t=s}^{e} f_t(\x_I^*) \bigg{\}} \\
& +  G_f  \sqrt{3\epsilon}  T .
\end{align*}
Using Lemma \ref{lemma:LOO-BOGD-WF}, we have that
\begin{align*}
    \sup\limits_{I=[s,e]\subseteq[T]} \bigg{\{} \sum_{t=s}^{e} f_t(\Tilde{\y}_{m(t)-1}) - \min\limits_{\x_I \in \mK} \sum_{t=s}^{e} f_t(\x_I) \bigg{\}}   \leq  4RG_fK + \frac{ 4R^2}{\eta}  + \frac{  G_f^2  }{2} K \eta T.
\end{align*}
Combining the last two equations and plugging-in the values of $\epsilon, \eta, K$ stated in the theorem, we obtain the adaptive regret bound stated in the theorem. 

We now move on to upper-bound the overall number of calls to the linear optimization oracle. Recall that on each block $m \in [2, \dots, T/K]$, the call to Algorithm \ref{alg:CIP-FW} returns points $(\x_m, \Tilde{\y}_m)\in\mK\times{}R\ball$ which satisfy $\Vert \x_{m} - \Tilde{\y}_{m} \Vert^2 \leq 3\epsilon $, and Algorithm \ref{alg:LOO-BOGD} updates $\y_{mK+1} = \Tilde{\y}_{m-1} - \eta \sum_{t=(m-1)K+1}^{mK}  \tilde{\nabla}_t$. 
Thus, the points  $\x_{m-1}, \y_{mK+1}$ which are the input sent to Algorithm \ref{alg:CIP-FW} on the following block $m+1$ satisfy:
\begin{align*}
    \Vert \x_{m-1} - \y_{mK+1} \Vert \leq \Vert \x_{m-1} - \Tilde{\y}_{m-1} \Vert + \Vert \Tilde{\y}_{m-1} - \y_{mK+1} \Vert \leq \sqrt{3 \epsilon} + K \eta G_f. 
\end{align*}
Using $(a+b)^2 \leq 2a^2 + 2b^2$, we have that for any block $m$,
\begin{align*}
    \Vert \x_{m-1} - \y_{mK+1} \Vert^2 \leq 6 \epsilon + 2 K^2 \eta^2 G_f^2.
\end{align*}
Using Lemma \ref{lemma:CIP-FW}, each call to Algorithm \ref{alg:CIP-FW} on some block $m$ makes at most
\begin{align*}
    \max \bigg{\{} \frac{\Vert \x_{m-1} - \y_{mK+1} \Vert^2 (\Vert \x_{m-1} - \y_{mK+1} \Vert^2 - \epsilon)}{4\epsilon^2}+1 , 1 \bigg{\}}
\end{align*}
iterations. On each iteration of Algorithm \ref{alg:CIP-FW} it calls  Algorithm \ref{alg:SH-FW}, which  according to Lemma \ref{lemma:SH-FW}, makes at most $\left\lceil \frac{27 R^2}{\epsilon } -2 \right\rceil$ calls to a linear optimization oracle. Thus, Algorithm \ref{alg:LOO-BOGD} on block $m$ makes
\begin{align*}
    n_m & \leq \max \bigg{\{} \frac{\Vert \x_{m-1} - \y_{mK+1} \Vert^2 (\Vert \x_{m-1} - \y_{mK+1} \Vert^2 - \epsilon)}{4\epsilon^2}+1  ,  1  \bigg{\}} \frac{27 R^2}{\epsilon }  \\
    & \leq \left( 8.5 +  5.5 \frac{   K^2 \eta^2 G_f^2 }{\epsilon} +  \frac{ K^4 \eta^4 G_f^4  }{\epsilon^2} \right) \frac{27 R^2}{\epsilon }
\end{align*}
calls to linear optimization oracle. Thus, the overall number of calls to a linear optimization oracle is 
\begin{align*}
    N_{calls}  = \sum_{m=1}^{T/K} n_m & \leq  \frac{T}{K} \left( 8.5 +   5.5 \frac{   K^2 \eta^2 G_f^2 }{\epsilon} +  \frac{ K^4 \eta^4 G_f^4  }{\epsilon^2} \right) \frac{27 R^2}{\epsilon }. 
\end{align*}
\end{proof}

\subsubsection{(standard) Regret bound for strongly convex losses}
We now consider the case in which all loss functions are $\alpha$-strongly convex, for some known $\alpha > 0$. In this setting, vanilla OGD does not yield adaptive regret guarantees, and the same goes for our OGD-based approach for constructing new LOO-based projection-free algorithms. Instead, here we show that our approach can recover, up to a logarithmic factor, the state-of-the-art (standard) regret bound for this setting of $O(T^{2/3})$ \cite{kretzu2021revisiting}, i.e. $\tilde{O}(T^{2/3})$. This result is obtained by using Algorithm \ref{alg:LOO-BOGD} with an appropriate choice of parameters.


\begin{theorem}\label{thm:SC-OGD-LOO} Suppose all loss functions $\{f_t\}_{t=1}^T$ are $\alpha$-strongly convex, for some $\alpha > 0$, and that $T\geq 27 (\alpha R/ G_f)^2$. Setting $\epsilon_m =  \left(\frac{20 G_f} { \alpha (m+3)}\right)^2 , \eta_m = \frac{2}{\alpha K m}$, for all $m\geq 1$, and  $K=\left( \frac{\alpha R}{G_f}\right)^\frac{2}{3} T^\frac{2}{3} $ in Algorithm \ref{alg:LOO-BOGD}, guarantees that the (static) regret is upper bounded by
\begin{align*}
    \sum_{t=1}^{T} f_t(\x_t) - \min_{\x \in \mK} \sum_{t=1}^{T} f_t(\x)  &  \leq  36 (G_f^4 R^2 / \alpha)^\frac{1}{3} T^\frac{2}{3} \left(1+ \frac{2}{3}  \ln{ \left(\sqrt{T}G_f/(\alpha R)\right)} \right) ,
\end{align*}
and that the overall number of calls to the linear optimization oracle is upper bounded by
\begin{align*}
    N_{calls}  
    \leq &  0.94  T.
\end{align*}
\end{theorem}

%

\begin{proof}[Proof of Theorem \ref{thm:SC-OGD-LOO}]
Recall that according to Lemma \ref{lemma:CIP-FW}, Algorithm \ref{alg:CIP-FW} is an infeasible projection oracle. Denote $m(t) = \lceil T/K \rceil$ and $ \Tilde{\nabla}_t \in \partial f_t(\Tilde{\y}_{m(t)-1})$ for all $t\in[T]$. 
Note that the sum of a $K$ $\alpha$-strongly convex functions, $\sum_{i=1}^{K} f_i(\cdot)$, is $\alpha K$-strongly convex. Since for every $t \in [T]$, $f_t(\cdot)$ is $\alpha$-strongly convex, and since for every $m \in [T/K]$ we have that $\eta_m = \frac{1}{\alpha K m}$, and $\Tilde{\y}_m$ is an infeasible projection of $\y_{(m-1)K+1}$ over $\mK$,  by applying Lemma \ref{lemma:OGD-WF} w.r.t. to prediction in blocks of length $K$ (as in Algorithm \ref{alg:LOO-BOGD}), it follows that for every $\x \in \mK$,
\begin{align*}
    \sum_{t=1}^{T} f_t(\Tilde{\y}_{m(t)-1}) - f_t(\x)  & \leq \sum_{m =1}^{T/K}  \frac{1}{2 \alpha K m} \left\Vert \sum_{t =(m-1)K+1}^{mK} \Tilde{\nabla}_t \right\Vert^2.
\end{align*}
Denote $\nabla_t \in \partial f_t(\x_{m(t)-1})$ for all $t\in[T]$. Since for every $t\in [T]$, $f_t(\cdot)$ is also $G_f-$Lipschitz, using Lemma \ref{lemma:CIP-FW}, we have that
\begin{align*}
    \sum_{t=1}^{T} f_t(\x_{m(t)-1}) - f_t(\Tilde{\y}_{m(t)-1}) \leq \sum_{m =1}^{T/K} \sum_{t =(m-1)K+1}^{mK} \nabla_t^\top (\x_{m-1} - \Tilde{\y}_{m-1}) \leq \sqrt{3} K G_f \sum_{m=2}^{T/K} \sqrt{\epsilon_{m-1}},
\end{align*}
where we have used the fact that $\x_0 = \tilde{\y}_0$. 

Denote $\x^* = \argmin_{\x \in \mK} \sum_{t=1}^{T} f_t(\x)$. Using Lemma \ref{lemma:CIP-FW} for every $m \in [T/K]$, it holds that $\Tilde{\y}_m \in R\ball$ and thus, it holds that $\Vert \Tilde{\nabla}_t \Vert \leq G_f$ for every $t\in[T]$. Combining the last two equations, we have
\begin{align*}
    \sum_{t=1}^{T} f_t(\x_{m(t)-1}) -  f_t(\x^*) & \leq  \sqrt{3} K G_f \sum_{m=1}^{T/K} \sqrt{\epsilon_{m}} + \frac{ K G_f^2 }{2 \alpha } \sum_{m=1}^{T/K} \frac{1}{m}.
\end{align*}
Plugging-in the values of $\{\epsilon_m\}_{m=1}^{T/K}, K$ listed in the theorem, and using the fact that $\sum_{m =1}^{T/K} m^{-1} \leq 1+ \ln{(T/K)}$, we obtain the regret bound listed in the theorem. 

We turn to upper-bound the number of calls to the linear optimization oracle. Recall that for every block $m\in[T/K]$, Algorithm \ref{alg:CIP-FW} returns the points $\x_m, \tilde{\y}_{m}$ such that $\Vert \x_m - \tilde{\y}_{m} \Vert^2 \leq 3\epsilon_m$, and that the update step in Algorithm \ref{alg:LOO-BOGD} is $\y_{(m+1)K+1} = \Tilde{\y}_{m} - \eta_{m+1} \sum_{t=mK+1}^{(m+1)K} \tilde{\nabla}_t$. Thus, for every $m \geq 1$ we have that,
\begin{align*}
    \Vert \x_{m-1} - \y_{mK+1} \Vert \leq \Vert \x_{m-1} - \Tilde{\y}_{m-1} + \Tilde{\y}_{m-1} - \y_{mK+1} \Vert \leq \sqrt{3\epsilon_{m-1}} + \eta_{m} K G_f.
\end{align*}
Using the inequality $(a+b)^2 \leq 2a^2 + 2b^2$, for any $m\geq 1$ we have that
\begin{align}
    \Vert \x_{m-1} - \y_{mK+1} \Vert^4 = \left( \Vert \x_{m-1} - \y_{mK+1} \Vert^2 \right)^2 \leq  72 \epsilon_{m-1}^2 + 8 \eta_{m}^4 K^4 G_f^4. \label{eq:SC-LOO-OGD-FW_dist_of_y_t}
\end{align}
Recall that for every block $m\in[T/K]$, Algorithm \ref{alg:CIP-FW} recives the points $\x_{m-2}, \y_{(m-1)K+1}$ and the tolerance $\epsilon_m$. Using Lemma \ref{lemma:CIP-FW}, each call to Algorithm \ref{alg:CIP-FW} on some block $m+1$ makes at most
\begin{align*}
    \max \bigg{\{} \frac{\Vert \x_{m-1} - \y_{mK+1} \Vert^2 \left( \Vert \x_{m-1} - \y_{mK+1} \Vert^2 - \epsilon_{m+1} \right)}{4\epsilon_{m+1}^2}+1 , 1 \bigg{\}}
\end{align*}
iterations. On each iteration $m+1$ of Algorithm \ref{alg:CIP-FW}, it calls  Algorithm \ref{alg:SH-FW}, which  according to Lemma \ref{lemma:SH-FW}, makes at most $\left\lceil \frac{27 R^2}{\epsilon_{m} } -2 \right\rceil$ calls to a linear optimization oracle. Thus, by using Eq. \eqref{eq:SC-LOO-OGD-FW_dist_of_y_t}, on block $m+1$, Algorithm \ref{alg:LOO-BOGD} makes
\begin{align*}
    n_{m+1} & \leq \max \bigg{\{} \frac{\Vert \x_{m-1} - \y_{mK+1} \Vert^4 }{4\epsilon_{m+1}^2}+1  ,  1  \bigg{\}} \frac{27 R^2}{\epsilon_{m+1} }  \leq  \left( \frac{18 \epsilon_{m-1}^2  }{\epsilon_{m+1}^2} + \frac{2 \eta_{m}^4 K^4 G^4 }{\epsilon_{m+1}^2}+1  \right) \frac{27 R^2}{\epsilon_{m+1} }
\end{align*}
calls to linear optimization oracle. Thus, by plugging-in the values for $\{ \epsilon_m  \}_{m=1}^{T/K}$, $\{ \eta_m \}_{m=1}^{T/K}$ and $K$ listed in the theorem, the overall number of calls to the linear optimization oracle is 
\begin{align*}
    N_{calls}  = \sum_{m=2}^{\frac{T}{K}} n_{m} & \leq \sum_{m=2}^{\frac{T}{K}}   \left( \frac{18 (m+3)^4    }{(m+1)^4} + \frac{ 2 (m+3)^4  }{400^2  (m-1)^4   }+1  \right)  \left(\frac{ R \alpha }{20  G_f} \right)^2 (m+3)^2 \\
& \leq  0.35  \left( R \alpha /  G_f \right)^2 \sum_{m=2}^{\frac{T}{K}}  (m+3)^2,
\end{align*}
where the last inequality is since $\frac{m+3}{m+1} \leq \frac{5}{3}$ and $\frac{m+3}{m-1} \leq 5$, for every $m\geq2$. 

Since $T\geq 3K$, it holds that $\sum_{m =5}^{T/K+3} m^{2}\leq \sum_{m =5}^{2T/K} m^{2}$. Thus, since $ \sum_{m =5}^{2T/K} m^{2} \leq \frac{8}{3} \left( \frac{T}{K} \right)^3 $, and by plugging-in $K = \left( \frac{\alpha R}{G_f}\right)^\frac{2}{3} T^\frac{2}{3}$, it indeed holds $N_{calls} \leq 0.94    T$.

\end{proof}

\subsection{LOO-based algorithm for the bandit setting}
Our algorithm for the bandit information setting using a LOO, Algorithm \ref{alg:LOO-BBGD} (given below),  follows from a simple combination of Algorithm \ref{alg:LOO-BOGD} and the standard technique for bandit optimization pioneered in \cite{Flaxman05}, which generates unbiased estimators for gradients of smoothed versions of the original (unknown) loss functions via random sampling in a small neighbourhood of the feasible point. For this reason, Algorithm \ref{alg:LOO-BBGD} applies the full-information Algorithm  \ref{alg:LOO-BOGD} on a slightly squeezed version of the feasible set --- the set $\mK_{\delta/r} = (1-\delta/r)\mK$, so that the sampled points will remain feasible. We remind the reader that in the bandit setting we make the standard assumption that the loss functions are chosen obliviously, i.e., they are independent of any randomness introduced by the algorithm.

\begin{algorithm}
\KwData{horizon $T$, feasible set $\mK$ with parameters $r,R$, block size $K$, step size $\eta$, smoothing parameter $\delta\in(0,r]$}
$\x_0,\x_1 \gets $ arbitrary points in $\mK_{\delta/r}$\\
$\Tilde{\y}_0 \gets \x_0$, $\y_1 \gets \Tilde{\y}_0, \Tilde{\y}_1 \gets \x_1$.\\
\For{$~ t = 1,\ldots,K ~$}{
    Set $\u_t$ $\sim S^n$ and play $\z_t = \Tilde{\x}_{0} + \delta \u_t$. \\
    Observe $f_{t}(\z_{t})$, set $\g_t = \frac{n}{\delta} f_t(\z_{t}) \u_t$ and update $\y_{t+1} = \y_{t} - \eta \g_t$.
    }
\For{$~ m = 2,\ldots,\frac{T}{K} ~$}{
    Let $(\x_{m},\tilde{\y}_{m})$ be the output  of  Algorithm \ref{alg:CIP-FW} with set $\mK_{\delta/r}$, feasible point $\x_{m-2}$, initial vector $\y_{(m-1)K+1}$, and tolerance $\frac{\delta^2}{3}$ (execute \textbf{in parallel} to  following \textbf{for} loop over $s$)\\
    Set $\y_{(m-1)K+1} = \Tilde{\y}_{m-1}$\\
    \For{$~ s = 1,\ldots,K ~$}{
        Set $\u_t$ $\sim S^n$ and play $\z_t = \Tilde{\x}_{m-1}+ \delta \u_t $. \tcc*{$t = (m-1)K+s$}
        Observe $f_{t}(\z_t)$, set $\g_t = \frac{n}{\delta} f_t(\z_{t}) \u_t$ and update $\y_{t+1} = \y_{t} - \eta \g_t$.  
        }
    \textbf{Note:} $\y_{mK+1} = \Tilde{\y}_{m-1} - \eta \sum_{t=(m-1)K+1}^{mK} \g_t$.
    }
\caption{Blocked Bandit Gradient Descent using Linear Optimization Oracle (LOO-BBGD) }\label{alg:LOO-BBGD}
\end{algorithm}

\begin{theorem}\label{thm:LOO-BBGD}   
Suppose Assumption \ref{ass:bandit} holds. For all $ c > 0 $ such that $\frac{cT^{-1/4}}{r} < 1$, Setting $\eta = \frac{R}{\sqrt{nM}}T^{-\frac{3}{4}},  K=6 nM T^{\frac{1}{2}}, \delta =  c T^{-\frac{1}{4}} $ in Algorithm \ref{alg:LOO-BBGD} guarantees that the adaptive expected regret is upper-bounded as follows
{\small \begin{align*}
    &AER_T = \sup\limits_{ I=[s,e]\subseteq[T]}  \Bigg{\{} \E \left[ \sum_{t=s}^{e} f_t(\z_{t}) \right] - \min\limits_{\x_I \in \mK} \sum_{t=s}^{e} f_t(\x_I)  \Bigg{\}} \leq  \\
    & ~~~~~~~~ \leq \left(4 +\frac{R}{r}\right) G_f c  T^{\frac{3}{4}} + \sqrt{nM} \left( 4R  +  \frac{1}{ \sqrt{6 }} + 3R G_f^2 +  \frac{ R nM}{ 2c^2} \right)  T^{\frac{3}{4}} +   24 R nM \left( \frac{\sqrt{nM}}{c \sqrt{6}} + G_f \right) T^{\frac{1}{2}},
\end{align*}}
and the expected overall number of calls to the linear optimization oracle is upper bounded by 
{\small \begin{align*}
    \E [N_{calls}] & \leq  \frac{27R^2}{ 2 nM c^2 } \left(  \frac{  6^5  R^4   \left(nM\right)^4 }{4  c^8} + \frac{   6^6 R^4    \left(nM\right)^3 G_f^2 }{ 2  c^6}  + \frac{  6^6  R^4  ( nM)^2 G_f^4 }{ 3  c^4}+19   \right)  T.
\end{align*} }
In particular, if $\left( \frac{20R\sqrt{nM}}{r} \right)^4 \leq T$ then, setting $c=20R \sqrt{nM}$, we have
{\small \begin{align*}
    & AER_T \leq   R \sqrt{nM} \left( 80  G_f + 20 G_f  \frac{R}{r} + 4  +  \frac{1}{ 2 R} + 3 G_f^2 +  \frac{1 }{ 4 R^2} \right)  T^{\frac{3}{4}} +   24  nM \left( \frac{1}{2 } + R G_f \right) T^{\frac{1}{2}},
\end{align*}}
and 
{\small\begin{align*}
    \E [N_{calls}] & \leq   \frac{\Tilde{c}}{  (nM)^2  } \left(  \frac{ 1 }{   R^4 } + \frac{ G_f^2 }{ R^2}  +   G_f^4 + 1   \right)  T,
\end{align*}}
where $0 < \Tilde{c} < 1$ is an universal constant. 
\end{theorem}

 Before proving Theorem \ref{thm:LOO-BBGD}, we need an additional lemma.

\begin{lemma} \label{lemma:expectation_gradient}
Fix some interval $\mT = \{\tau+1, \dots, \tau + L\}$ of size $L$, a set of i.i.d. samples $ \{ \u_t\}_{t \in \mT}, \u_t\sim\mS^n$, and some $\y \in (1-\delta/r) \mK = \mK_{\delta/r}$, for some $\delta\in(0,r)$. Define $\g_t = \frac{n}{\delta} f_t(\y + \delta \u_t) \u_t,  t \in \mT$, and let $\widehat{\g}_{\mT} = \sum\limits_{t \in \mT} \g_t$. Then, it holds that
\begin{enumerate}
    \item $\E \left[ \Vert \widehat{\g}_\mT \Vert  \right]^2 \leq \E \left[  \Vert \widehat{\g}_\mT \Vert ^2 \right] \leq L \left(\frac{n M}{\delta}\right)^2 + L^2 G_f^2$.
    \item $\E \left[ \Vert \widehat{\g}_{\mT} \Vert ^4 \right] 
    \leq  3 L^2 \left(\frac{nM}{\delta}\right)^4  + 6L^3 \left(\frac{nM}{\delta}\right)^2 G_f^2 + L^4 G_f^4$.
\end{enumerate}
\end{lemma}

\begin{proof}
We start with the first item. It holds that
\begin{align}
    \E \left[ \Vert \widehat{\g}_{\mT} \Vert ^2 \right] & = \E \left[ \left\Vert \sum_{t\in\mT} \g_t \right\Vert ^2 \right]  =  \E \left[\sum_{t\in\mT}  \Vert \g_t \Vert ^2 + \sum_{(i,j)\in\mT^2,i\neq j}\g_i^{\top} \g_j \right] \nonumber \\
    & = \E \left[\sum_{t\in\mT}  \Vert  \g_t \Vert ^2 \right] + \sum_{(i,j)\in\mT^2,i\neq j} \E \left[ \g_i^{\top} \g_j \right] . \label{eq:second_moment_block_gradient_estimator}
\end{align}	
Since $\max_{\x \in \mK}  \vert f_t(\x) \vert  \leq M$, we  have that $  \Vert \g_t \Vert  \leq \frac{n}{\delta}  \vert f_t(\y + \delta \u_t) \vert   \Vert \u_t \Vert  \leq \frac{nM}{\delta}$, and thus,
\begin{align}
    \sum_{t\in\mT} \Vert  \g_t \Vert ^2  \leq  L \left(\frac{n M}{\delta}\right)^2. \label{eq:upper_bound_of_dependent_gradient_estimator}
\end{align}

Using Lemma \ref{lemma:bertsekas_grdient} we have that for all $t\in\mT$, $\left\Vert{\E [\g_t|\y]}\right\Vert = \left\Vert{\nabla {\widehat{f}}_{t,\delta} (\y)}\right\Vert \leq G_f$. Furthermore, since, conditioned on $\y$, $\forall i \neq j$,  $\g_i$, $\g_j$ are independent random vectors, we have that
\begin{align}
    \sum_{(i,j)\in\mT^2,i\neq j} \E \left[ \g_i^{\top} \g_j \right] = \sum_{(i,j)\in\mT^2,i\neq j}\E\left[{ \E [ \g_i|\y]^{\top} \E [\g_j |\y]}\right] \leq \left(L^2 - L \right) G_f^2. \label{eq:upper_bound_of_independent_gradient_estimator}
\end{align}
Combining Equations \eqref{eq:second_moment_block_gradient_estimator}, \eqref{eq:upper_bound_of_dependent_gradient_estimator}, and \eqref{eq:upper_bound_of_independent_gradient_estimator}, we obtain the first part of the lemma:
\begin{align*}
  \E \left[  \Vert \widehat{\g}_\mT \Vert  \right]^2  \leq  \E \left[ \Vert \widehat{\g}_{\mT} \Vert ^2 \right]  \leq  L \left(\frac{n M}{\delta}\right)^2 + \left(L^2 - L \right) G_f^2,
\end{align*}
where the first inequality follows from Jensen's inequality. 

We move on to prove the second part of the lemma. It holds that 
{\small \begin{align*}
    & \E \left[ \Vert \widehat{\g}_{\mT} \Vert ^4 \right]  = \E \left[ \left\Vert \sum\limits_{t \in \mT_m}  \g_t \right\Vert^4 \right]  = \E \left[ \left( \sum_{t\in\mT}  \Vert \g_t \Vert^2  + \sum_{(i,j)\in\mT^2,i\neq j}\g_i^{\top} \g_j \right)^2 \right] \\
    & ~~~~~~~~ = \E \left[ \left( \sum_{t\in\mT} \Vert \g_t \Vert^2 \right)^2 \right] + 2 \E \left[ \left( \sum_{t\in\mT}  \Vert \g_t \Vert^2 \right) \left( \sum_{(i,j)\in\mT^2,i\neq j}\g_i^{\top} \g_j \right) \right] + \E \left[ \left( \sum_{(i,j)\in\mT^2,i\neq j}\g_i^{\top} \g_j \right)^2 \right].
\end{align*}}
Using Eq. \eqref{eq:upper_bound_of_dependent_gradient_estimator} and Eq. \eqref{eq:upper_bound_of_independent_gradient_estimator} we have,
\begin{align*}
    \E \left[ \Vert \widehat{\g}_{\mT} \Vert ^4 \right] &\leq  L^2 \left(\frac{nM}{\delta}\right)^4  + 2  L \left(\frac{nM}{\delta}\right)^2    \sum_{(i,j)\in\mT^2,i\neq j} \E \left[ \g_i^{\top} \g_j  \right] + \E \left[ \left( \sum_{(i,j)\in\mT^2,i\neq j}\g_i^{\top} \g_j \right)^2 \right] \\
    &\leq  L^2 \left(\frac{nM}{\delta}\right)^4  + 2  L \left(\frac{nM}{\delta}\right)^2  (L^2 -L)G_f^2 + \E \left[ \left( \sum_{(i,j)\in\mT^2,i\neq j}\g_i^{\top} \g_j \right)^2 \right] .
\end{align*}
Now we upper-bound the last term in the RHS. Note that, the expectation argument has $(L^2-L)^2$ summands. Since conditioned on $\y$, for every four indices $i\neq j\neq k\neq l $,  the random vectors $\g_i$, $\g_j$, $\g_k$, $\g_l$ are independent, we have that
\begin{align}\label{eq:lem:moments:1}
   \E \left[ \g_i^{\top} \g_j \g_k^{\top} \g_l \right] =  \E\left[{ \E [ \g_i^{\top}|\y] \E [\g_j |\y]\E [ \g_k^{\top}|\y] \E [\g_l |\y]}\right] \leq G_f^4,
\end{align}
where the last inequality follows, as before, from Lemma \ref{lemma:bertsekas_grdient} which yields $\left\Vert{\E [\g_t|\y]}\right\Vert  \leq G_f$ for all $t\in\mT$. 

In the case of three different indices $i\neq j\neq k$,  we have
\begin{align}\label{eq:lem:moments:2}
    \E \left[ \g_j^{\top} \g_i \g_k^{\top} \g_j \right] & = \E \left[ \g_j^{\top} \g_i \g_j^{\top} \g_k \right]  = \E \left[ \g_i^{\top} \g_j \g_k^{\top} \g_j \right] =  \E \left[ \g_i^{\top} \g_j \g_j^{\top} \g_k \right] \nonumber \\ 
    & = \E\left[ \E [ \g_i|\y]^{\top} \E [\g_j \g_j^{\top}|\y] \E [\g_k |\y] \right] \leq  \E\left[ \Vert \E [ \g_i|\y] \Vert \Vert \E [\g_j \g_j^{\top}|\y] \Vert \Vert \E [\g_k |\y] \Vert \right] \nonumber \\
    & \leq \E  [\Vert \g_j \Vert^2 ] G_f^2 \leq \frac{n^2M^2G_f^2}{\delta^2}.
\end{align}
There are $L(L-1)(L-2)(L-3)$ summands with four  different indices, and $2(L^2-L)$ summands with exactly two different indices. Thus, since there are overall $(L^2-L)^2$ summands, there are $4L^3-12L^2+8L$ summands wiht exactly three different indices. Thus, using Lemma \ref{lemma:bertsekas_grdient}, Eq. \eqref{eq:lem:moments:1}, and Eq. \eqref{eq:lem:moments:2},  it holds that
\begin{align*}
    \E \left[ \left( \sum_{(i,j)\in\mT^2,i\neq j}\g_i^{\top} \g_j \right)^2 \right] \leq 2L^2 \left(\frac{nM}{\delta}\right)^4 + 4L^3 \left(\frac{nM}{\delta}\right)^2 G_f^2 + L^4 G_f^4.
\end{align*}
Thus,  we obtain that
\begin{align*}
    \E \left[ \Vert \widehat{\g}_{\mT} \Vert ^4 \right] 
    &\leq  3 L^2 \left(\frac{nM}{\delta}\right)^4  + 6L^3 \left(\frac{nM}{\delta}\right)^2 G_f^2 + L^4 G_f^4.
\end{align*}
\end{proof}

\begin{proof} [Proof of Theorem \ref{thm:LOO-BBGD}]
First, we establish that Algorithm \ref{alg:CIP-FW} indeed plays feasible points.  Using Lemma \ref{lemma:CIP-FW}, for each block $m \in [2, \dots, T/K]$, Algorithm \ref{alg:CIP-FW} returns $\x_{m} \in \mK_{\delta/r} = (1-\delta/r)\mK$. Thus, for every iteration $t \in [T]$ it indeed holds that $\z_t \in \mK$. 

We now turn to prove  the upper-bound to the adaptive expected regret. Throughout the proof of the regret bound let us fix some interval $I=[s,e], 1\leq s\leq e\leq T$. We start with an upper bound on $\E \left[ \sum_{t=s}^{e} \widehat{f}_{\delta,t} \left(\x_{m(t)-1}\right) - \widehat{f}_{\delta,t} (\x) \right]$ which holds for every $\x \in \mK_{\delta/r}$. We will first take a few preliminary steps. For all $t \in \left[T\right]$, denote the history of all predictions and gradient estimates by $\mathcal{F}_t = \{ \x_1,  \dots, \x_{m(t-1)}, \g_1, \dots, \g_{t-1} \}$, where $m(t) := \left\lceil \frac{t}{K} \right\rceil$. Since for all $t\in[T]$, $\g_t$ is an unbiased estimator of $\nabla {\widehat{f}}_{t,\delta} \left(\x_{m(t)-1}\right)$, i.e., $\E \left[ \g_t|\mathcal{F}_t \right] = \nabla {\widehat{f}}_{t,\delta} \left(\x_{m(t)-1}\right)$, and $\E \left[ \x_{m(t)-1} | \mathcal{F}_t \right] = \x_{m(t)-1}$, we have that for every  $t\in[T]$ and $\x \in \mK_{\delta/r}$ it holds that,
\begin{align}
    \E \left[ \g_t^{\top} \left(\x_{m(t)-1} - \x\right) \right] & = \E \left[ \E \left[  \g_t | \mathcal{F}_{t} \right] ^{\top}   (\x_{m(t)-1} -   \x) \right] =  \E \left[\nabla \widehat{f}_{\delta,t} (\x_{m(t)-1})^\top (\x_{m(t)-1} -   \x) \right]. \label{eq:bf_unbiased_gradient_esimator_loo}
\end{align}
For every block $m \in [T/K]$, denote $\mT_m = \{ (m-1)K+1, \dots, mK \}$. Using Lemma \ref{lemma:CIP-FW}, we have that for every block $m \in [T/K]$, the point $\Tilde{\y}_{m+2}$ is an infeasible projection of $\y_{(m+1)K+1}$ over $\mK_{\delta/r}$. Since $\y_{(m+1)K+1} = \Tilde{\y}_{m} -  \eta \sum_{t \in \mT_{m+1}} \g_t$, we have that for every block $m\in[T/K]$ and $\x \in \mK_{\delta/r}$, it holds that
\begin{align*}
    \Vert \Tilde{\y}_{m+2}  - \x \Vert^2   & \leq \Vert  \y_{(m+1)K+1} - \x \Vert^2 = \left\Vert  \Tilde{\y}_{m} -  \eta \sum\nolimits_{t \in \mT_{m+1}}  \g_t  -\x \right\Vert^2 \\
    & = \left\Vert   \Tilde{\y}_{m} - \x \right\Vert^2  +   \eta^2 \left\Vert \sum\nolimits_{t \in \mT_{m+1}}  \g_t \right\Vert^2 - 2 \eta  \sum\nolimits_{t \in \mT_{m+1}}  \g_t^\top (\Tilde{\y}_{m} - \x).
\end{align*}
Rearranging, we have for every block $m$ that,
\begin{align}
    \sum\nolimits_{t \in \mT_{m+1}}  \g_t^\top  (\Tilde{\y}_{m} - \x)  \leq  \frac{ \left\Vert   \Tilde{\y}_{m} - \x \right\Vert^2}{2\eta} - \frac{ \Vert \Tilde{\y}_{m+2} - \x \Vert^2}{2\eta} + \frac{\eta}{2} \left\Vert \sum\nolimits_{t \in \mT_{m+1}}  \g_t \right\Vert^2. \label{eq:LOO-BBGD_one_iteration_bound}
\end{align}
Denote by $m_s$ and $m_e$ the smallest and the largest index of block that is fully contained in the interval $[s,e]$, respectively, i.e., $\{ (m_s-1)K+1, \dots , m_e K \} = \{ \mT_{m_s}, \dots, \mT_{m_e} \} \subseteq [s,e]$. Recall that  all iterations $t \in \mT_{m}$ share the same prediction $\Tilde{\y}_{m-1}$. Since $\{s, \dots, m_{s-1}K \} \subset \mT_{m_{s-1}}$ and $\{m_{e}K+1, \dots, e \} \subset \mT_{m_{e+1}}$, for every $\x \in \mK_{\delta/r}$ we have that
\begin{align*}
    \E \left[ \sum\limits_{t =s}^{e}  \g_t^\top  \left(\Tilde{\y}_{m(t)-1} - \x\right) \right]  = & \E \left[  \sum\limits_{t =s}^{m_{s-1}K} \g_t^\top  (\Tilde{\y}_{m_{s-2}} - \x) \right] + \E \left[  \sum\limits_{m =m_s}^{m_e}  \sum\limits_{t \in \mT_{m}}  \g_t^\top  (\Tilde{\y}_{m-1} - \x) \right] \\ 
    &  
    + \E \left[  \sum\limits_{t =m_e K +1}^{e}  \g_t^\top  (\Tilde{\y}_{m_{e}} - \x) \right].
\end{align*}
Using the Cauchy-Schwarz inequality, Lemma \ref{lemma:CIP-FW} (which yields that $\Tilde{\y}_m \in R \ball$), and Lemma \ref{lemma:expectation_gradient}, with the fact that for all $a,b \in \reals^+ : ~ \sqrt{a+b} \leq \sqrt{a} + \sqrt{b}$, we obtain the bound: $\E \left[  \sum_{t =m K +1}^{(m+1)K}  \g_t^\top  (\Tilde{\y}_{m(t)} - \x) \right] \leq 2R K \left( \frac{nM}{\delta \sqrt{K}}+ G_f \right) $ for every $t\in[T]$, and $\x \in \mK_{\delta/r}$. Combining Eq.\eqref{eq:LOO-BBGD_one_iteration_bound}, and this bound, we have that
\begin{align*}
    \E \left[ \sum\limits_{t =s}^{e}  \g_t^\top  \left(\Tilde{\y}_{m(t)-1} - \x\right) \right] \leq 4 R K \left( \frac{nM}{\delta \sqrt{K}} + G_f \right) + \frac{4R^2}{\eta} + \frac{\eta  }{2}   \sum\limits_{m =m_s}^{m_e} \E \left[ \left\Vert \sum\limits_{t \in \mT_{m}}  \g_t \right\Vert^2 \right].
\end{align*}
Combining Eq.\eqref{eq:bf_unbiased_gradient_esimator_loo} and Lemma \ref{lemma:expectation_gradient}, we have that for every $\x \in \mK_{\delta/r}$ it holds that,
\begin{align*}
    \sum_{t=s}^{e} \E \left[ \nabla \widehat{f}_{\delta,t} \left(\x_{m(t)-1}\right)^\top \left(\x_{m(t)-1} - \x\right) \right] \leq  & \E \left[ \sum\limits_{t =s}^{e}  \g_t^\top  \left( \x_{m(t)-1} - \Tilde{\y}_{m(t)-1} \right) \right] + \frac{4 R^2}{\eta}  \nonumber\\
    &  + 4R K \left( \frac{nM}{\delta \sqrt{K}} + G_f\right) + \frac{\eta  }{2} K \left( \frac{ n^2 M^2}{\delta^2 K} +  G_f^2 \right) T. 
\end{align*}
From Lemma \ref{lemma:CIP-FW}, we have that for every block $m$, $\Vert \x_{m-1} - \Tilde{\y}_{m-1} \Vert \leq \delta$. Using Lemma    \ref{lemma:expectation_gradient} with the fact that for all $a,b \in \reals^+ : ~ \sqrt{a+b} \leq \sqrt{a} + \sqrt{b}$, we have that $\E \left[  \Vert \sum_{t \in \mT_{m}}   \g_t \Vert \right] \leq \sqrt{K} \left(\frac{n M}{\delta}\right) + K G_f $ for every block $m$. Plugging-in these two observations, we have that for every $\x \in \mK_{\delta/r}$ it holds that,  
\begin{align*}
    \sum_{t=s}^{e} \E \left[ \nabla \widehat{f}_{\delta,t} \left(\x_{m(t)-1}\right)^\top \left(\x_{m(t)-1} - \x\right) \right] \leq  &  \left( \frac{n M}{ \sqrt{K }}  + \delta G_f \right) T + 4R K \left( \frac{nM}{\delta \sqrt{K}} + G_f\right)\\
    & + \frac{4R^2}{\eta} + \frac{\eta  }{2} K \left( \frac{ n^2 M^2}{\delta^2 K} +  G_f^2 \right) T. 
\end{align*}
Since for every $t \in [T]$, $ f_t (\cdot)$ is convex in $\mK $, using Lemma \ref{lemma:hazan_smooth} it holds that the smoothed function $\widehat{f}_{t,\delta} (\cdot)$ is convex in $\mK_{\delta/r}$ for all $t\in[T]$. Thus, for every $\x \in \mK_{\delta/r}$ we obtain that,
\begin{align}
    \E \left[ \sum_{t=s}^{e} \widehat{f}_{\delta,t} \left(\x_{m(t)-1}\right) - \widehat{f}_{\delta,t} (\x) \right] \leq & \left( \frac{n M}{ \sqrt{K }}  + \delta G_f \right) T + 4R K\left( \frac{nM}{\delta \sqrt{K}} + G_f \right) \nonumber \\
    &  + \frac{4R^2}{\eta} + \frac{\eta K }{2} \left( \frac{ n^2 M^2}{K \delta^2} + G_f^2 \right) T. \label{eq:LOO-BBGD-FW_regret_infeasible_projection}
\end{align}
Let us now denote by $\x_I^*$ a feasible minimizer w.r.t. the interval $I=[s,e]$, i.e. $\x_I^* \in \argmin\limits_{\x \in \mK} \sum_{t=s}^{e} f_t(\x)$. Denote also $\Tilde{\x}_{I}^* = \left( 1-\frac{\delta}{r} \right) \x_I^* \in \mK_{\delta/r} $. It holds that,
{\small \begin{align}
    \E \left[\sum_{t=s}^{e} f_t(\z_{t})  -  f_t(\x_I^*) \right] = & \E \left[ \sum_{t=s}^{e} f_t(\z_{t}) - f_t\left(\x_{m(t)-1}\right) + f_t\left(\x_{m(t)-1}\right) - f_t(\Tilde{\x}_{I}^*)    +   f_t(\Tilde{\x}_{I}^*) -  f_t(\x_I^*) \right]. \label{eq:LOO-BBGD-FW_full_regret}
\end{align}}
Since for every $t \in [T]$, $\z_t = \x_{m(t)-1} + \delta \u_t$, and $f_t$ is $G_f$-Lipschitz over $\mK$, we have that
\begin{align*}
    \E \left[ \sum_{t=s}^{e} f_t(\z_{t}) - f_t(\x_{m(t)-1}) \right] & =  \sum_{t=s}^{e} \E \left[ f_t(\x_{m(t)-1} + \delta \u_t) - f_t(\x_{m(t)-1}) \right]  \\
& \leq \sum_{t=s}^{e} \E \left[G_f  \delta \Vert \u_t \Vert \right] \leq  G_f \delta T,
\end{align*}
and since $\Vert \x_I^* \Vert \leq R$, we have
\begin{align*}
    \E \left[ \sum_{t=s}^{e} f_t(\Tilde{\x}_{I}^*) -  f_t(\x_I^*) \right] = \sum_{t=s}^{e} f_t(\Tilde{\x}_{I}^*) -  f_t(\x_I^*) \leq \sum_{t=s}^{e} G_f  \Vert \Tilde{\x}_{I}^* - \x_I^* \Vert \leq  \frac{R G_f}{r} \delta T.
\end{align*}
Using Lemma \ref{lemma:hazan_smooth} and Eq.\eqref{eq:LOO-BBGD-FW_regret_infeasible_projection}, we have
\begin{align*}
    & \E \left[ \sum_{t=s}^{e} f_t\left(\x_{m(t)-1}\right) - f_t(\Tilde{\x}_{I}^*) \right] =  \E \left[ \sum_{t=s}^{e} f_t\left(\x_{m(t)-1}\right) - \widehat{f}_{\delta,t}\left(\x_{m(t)-1}\right) \right]  \\
    &~~~~~~~~~~~~~~~~~~~~~~~~~~~~~~~~~~~~~~~  + \E \left[ \sum_{t=s}^{e} \widehat{f}_{\delta,t} \left(\x_{m(t)-1}\right) - \widehat{f}_{\delta,t} (\Tilde{\x}_{I}^*) \right] + \E \left[ \sum_{t=s}^{e} \widehat{f}_{\delta,t}(\Tilde{\x}_{I}^*) - f_t(\Tilde{\x}_{I}^*) \right] \\
    &~~~~~~~~ \leq  2 \delta G_f T + \left( \frac{n M}{ \sqrt{K }}  + \delta G_f \right) T + 4 R K \left( \frac{nM}{\delta \sqrt{K}} + G_f\right) + \frac{4R^2}{\eta} + \frac{\eta  }{2} K \left( \frac{ n^2 M^2}{\delta^2 K} +  G_f^2 \right) T.
\end{align*}
Combining the last three equations and Eq.\eqref{eq:LOO-BBGD-FW_full_regret}, we obtain that 
{ \begin{align*}
    \E \left[\sum_{t=s}^{e} f_t(\z_{t})  -   f_t(\x_I^*) \right]  \leq & \left(3+\frac{R}{r}\right) G_f \delta  T + \left( \frac{n M}{ \sqrt{K }}  + \delta G_f \right) T  + 4 R K \left( \frac{nM}{\delta \sqrt{K}} + G_f \right)\\ 
    & + \frac{4R^2}{\eta}  + \frac{\eta  }{2} \left( \frac{ n^2 M^2}{\delta^2} + K G_f^2 \right) T.
\end{align*}}
Plugging-in the values of $K, \eta, \delta$ listed in the theorem, we obtain the regret bound of the theorem.

We now turn to prove the upper-bound on the expected overall number of calls to the LOO. We start with find an upper-bound on $\E \left[ \Vert \x_{m-1} - \y_{mK+1} \Vert^4 \right]$. Since $(a+b)^4 \leq 8(a^4 + b^4)$, we have
\begin{align*}
    \E \left[ \Vert \x_{m-1} - \y_{mK+1} \Vert^4 \right] & = \E \left[  \Vert \x_{m-1} - \Tilde{\y}_{m-1} + \Tilde{\y}_{m-1} - \y_{mK+1} \Vert^4 \right] \\
    & \leq 8\E \left[  \Vert \x_{m-1} - \Tilde{\y}_{m-1}\Vert^4  + \Vert \Tilde{\y}_{m-1} - \y_{mK+1} \Vert^4 \right].
\end{align*}
Using Lemma \ref{lemma:CIP-FW}, for every block $m$, Algorithm \ref{alg:CIP-FW} returns points $\x_{m},\Tilde{\y}_{m}$ such that $ \Vert \x_{m} - \Tilde{\y}_{m} \Vert^2 \leq \delta^2 $. Since Algorithm \ref{alg:LOO-BBGD} updates $\y_{mK+1} = \Tilde{\y}_{m-1} - \eta \sum_{t \in \mT_m}  \g_t$, using Lemma \ref{lemma:expectation_gradient}, we have that
\begin{align*}
    \E \left[ \Vert \x_{m-1} - \y_{mK+1} \Vert^4 \right] & \leq 8 \left( \delta^4 + \eta^4 \E \left[ \left\Vert \sum\limits_{t \in \mT_m}  \g_t \right\Vert^4 \right] \right) \\
    & \leq  8 \left( \delta^4 +  3 \eta^4 K^2 \left(\frac{nM}{\delta}\right)^4  + 6 \eta^4 K^3 \left(\frac{nM}{\delta}\right)^2 G_f^2 + \eta^4 K^4 G_f^4\right).
\end{align*}
Using Lemma \ref{lemma:CIP-FW}, for every block $m$, Algorithm \ref{alg:CIP-FW} makes at most
{\small \begin{align*}
    \max \Bigg{\{} \frac{ \Vert \x_{m-1} - \y_{mK+1} \Vert^2 \left (\Vert \x_{m-1} - \y_{mK+1} \Vert^2 - \frac{\delta^2}{3} \right) }{ 4\left(\frac{\delta^2}{3}\right)^2}+1 , 1 \Bigg{\}}
\end{align*}}
iterations, where $\delta^2/3$ is the error tolerance. On each iteration of Algorithm \ref{alg:CIP-FW}, it calls  Algorithm \ref{alg:SH-FW}, which in turn, by Lemma \ref{lemma:SH-FW}, makes at most $\left\lceil \frac{27R^2}{\delta^2 / 3 }-2 \right\rceil$ calls to a linear optimization oracle. Thus, the call to Algorithm \ref{alg:CIP-FW} in block $m$ executes
\begin{align*}
    \E [ n_m ] & \leq \E \left[ \frac{ \Vert \x_{m-1} - \y_{mK+1} \Vert^2 \left (\Vert \x_{m-1} - \y_{mK+1} \Vert^2 - \frac{\delta^2}{3} \right) }{ 4\left(\frac{\delta^2}{3}\right)^2}+1 \right] \frac{81R^2}{ \delta^2 }  \\
    & \leq  \left( \frac{18 \eta^4 K^2 \left(  3 \left(\frac{nM}{\delta}\right)^4  + 6K \left(\frac{nM}{\delta}\right)^2 G_f^2 + K^2 G_f^4\right) }{ \delta^4}+19 \right) \frac{81R^2}{ \delta^2 }
\end{align*}
calls to linear optimization oracle in expectation. Thus, the expected overall number of calls to a linear optimization oracle is bounded by
{\small \begin{align*}
    \E \left[ N_{calls} \right]  = \sum_{m=1}^{T/K} \E \left[  n_m \right] & \leq  \frac{T}{K } \left(  \frac{ 54 \eta^4  K^2   \left(nM\right)^4 }{ \delta^8} + \frac{ 108 \eta^4  K^3 \left(nM\right)^2 G_f^2 }{ \delta^6}  + \frac{ 18 \eta^4 K^4 G_f^4 }{ \delta^4}+19   \right) \frac{81R^2}{ \delta^2 }.
\end{align*}}
It only remains to plug-in the value of $K, \eta,\delta$ listed in the theorem.
\end{proof}

\section{Projection-free Algorithms via a Separation Oracle}
In this section we discuss our SO-based algorithms.
Similarly to our LOO-based algorithms, here also we will begin by showing how to efficiently compute infeasible projections using the SO, and then we will combine it with the OGD without feasibility approach (Algorithm \ref{alg:OGD-WF}), to obtain our algorithms. More concretely, our SO-based algorithms  will be based on the following idea, which is slightly different than the one used for our LOO-based algorithms. Note that under Assumption \ref{ass:bandit}, for any $\delta\in[0,1]$ it holds that $\mK_{\delta} = (1-\delta)\mK\subseteq\mK$. Thus, our approach will be to fix some $\delta\in(0,1]$ and to treat $\mK_{\delta}$ as if it was the feasible set, and compute infeasible projections w.r.t. to it, while ensuring that at all times, the points played by the algorithms remain within the enclosing feasible set $\mK$. 

For clarity, throughout this section we introduce the notation $\mK_{\delta_1,\delta_2} = (1-\delta_1)(1-\delta_2)\mK = \{(1-\delta_1)(1-\delta_2)\x~|~\x\in\mK\}$, for any $(\delta_1,\delta_2)\in[0,1]^2$.

\subsection{Efficient (close) infeasible projection via a SO}\label{sec:so-oracle}

We now turn to detail the main ingredient in our SO-based online algorithms --- efficient infeasible projections onto the set $\mK_{\delta,\delta'/r}$, for any given $(\delta,\delta')\in[0,1]\times[0,r]$, using the SO. 

As in our LOO-based construction, the first step will be to show how the SO of $\mK$ can be used to construct separating hyperplanes w.r.t. $\mK_{\delta,\delta'/r}$, which will in turn be used to ``pull'' infeasible points closer to the set, while maintaining the infeasible projection property.

\begin{lemma} \label{lemma:SH} Suppose Assumption \ref{ass:bandit} holds. Fix $(\delta,\delta')\in(0,1)\times[0,r)$, and let $\y\in\reals^n$ such that $\frac{\y}{1-\delta'/r} \notin \mK$. Let $\g\in\reals^n$ be the output of the SO of $\mK$ w.r.t. $\frac{\y}{1-\delta'/r}$, i.e., for all $\x \in \mK$, $\left(\frac{\y}{1-\delta'/r} -\x\right)^\top \g > 0$. Then, it holds that, 
\begin{align*}
     \forall \z \in \mK_{\delta, \delta'/r}   :\quad (\y - \z)^\top \g >  \delta (r-\delta') \Vert \g \Vert.
\end{align*}
\end{lemma}

Before proving  the lemma we require an additional observation.
\begin{observation}\label{obs:radius_ball_in_k_delta}
Suppose Assumption \ref{ass:bandit} holds and fix some $(\delta,\delta')\in[0,1]\times[0,r]$. Then, for all $\z\in\mK_{\delta,\delta'/r}=(1-\delta)(1-\delta'/r)\mK$, it holds that $\z+\delta(r-\delta')\ball\subseteq\mK_{\delta'/r}$.
\end{observation}

\begin{proof}[Proof of Lemma \ref{lemma:SH}]
Note that $\mK_{\delta} = (1-\delta)\mK  \subseteq \mK$, and $\mK_{\delta, \delta'/r} = (1-\delta'/r)(1-\delta)\mK  \subseteq \mK_{\delta}$. Since for all  $\x \in \mK$, $ (\y - (1-\delta'/r) \x)^\top \g  > 0$, we have that  for all $\w \in \mK_{\delta'/r}$, $(\y -\w)^\top \g > 0$. Fix some $\z\in\mK_{\delta,\delta'/r}$, and note that using Observation \ref{obs:radius_ball_in_k_delta}, it holds that $\z + \delta (r-\delta') \hat{\g} \in \mK_{\delta'/r}$, where $ \hat{\g} = \frac{\g}{\Vert \g \Vert}$. Then, we have that,
\begin{align*}
   \quad 0 & < \left( \y - (\z + \delta (r-\delta') \hat{\g}) \right)^\top \g  = ( \y - \z)^\top \g  - \delta (r-\delta') \Vert \g \Vert.
\end{align*}
Rearranging, we obtain the lemma.
\end{proof}

We can now present our SO-based infeasible projection oracle, see Algorithm  \ref{alg:CIP-SO}.

\begin{algorithm}[!]
  \KwData{feasible set $\mK$, radius $r$, squeeze parameters $(\delta,\delta')\in[0,1]\times[0,r]$, initial vector $\y_{0}$.}
  $\y_{1} \gets \y_{0} / \max \{ 1 , \Vert \y \Vert / R \} $ \tcc*{$\y_{1}$ is projection of $\y_{0}$ over $R\ball$}
  \For{$i=1 \dots$}{
   Call  $\textrm{SO}_{\mK}$ with input $\frac{\y_{i}}{1-\delta'/r}$\\
    \eIf{$\frac{\y_{i}}{1-\delta'/r} \notin \mK$}{
        Set $\g_i \gets$  hyperplane outputted by $\textrm{SO}_{\mK}$ \tcc*{$\forall \x\in\mK~\left({\frac{\y_{i}}{1-\delta'/r}-\x}\right)^{\top}\g_i > 0$} 
        Update $\y_{i+1} = \y_{i} - \gamma_i \g_i$
    }{
        \textbf{Return} $\y \gets \y_{i}$
    }
  }
  \caption{Infeasible projection via a separation oracle}\label{alg:CIP-SO}
\end{algorithm}

\begin{lemma} \label{lemma:CIP-SO}
   Suppose Assumption \ref{ass:bandit} holds. Let $0 < \delta < 1$, and $0 \leq \delta' < r$.  Setting $\gamma_i = \delta  (r-\delta') /  \Vert \g_i \Vert$, Algorithm \ref{alg:CIP-SO} stops after at most $\frac{ \dist^2(\y_0, \mK_{\delta,\delta'/r}) - \dist^2(\y, \mK_{\delta,\delta'/r})}{ \delta^2 (r-\delta')^2}+1  $, iterations,  and returns $\y \in \mK_{\delta'} = (1-\delta')\mK$ such that 
    \begin{align*}
        \forall \z \in \mK_{\delta,\delta'/r}: \quad  \Vert \y - \z \Vert^2 \leq  \Vert \y_{0} - \z \Vert^2.
    \end{align*}
\end{lemma}

\begin{proof}
Denote by $k$ the number of iterations until Algorithm \ref{alg:CIP-SO} stops. Then, for every iteration $i < k$, it holds that $\frac{\y_{i}}{1-\delta'/r} \notin \mK$, which implies that $\y_{i} \notin \mK_{\delta'/r} = (1-\delta'/r) \mK$. Thus, using Lemma \ref{lemma:SH}, we have that for every $i<k$, it holds for all $\z \in \mK_{\delta, \delta'/r}$ that $\left( \y_{i} - \z \right)^\top \g_i \geq \delta (r-\delta') \Vert \g_i \Vert$. From these observations and using Lemma \ref{lemma:update_step_with_hp} with $\g = \g_i$, $C = \Vert \g_i \Vert$, and $Q = \delta (r-\delta') \Vert \g_i \Vert$, we have that for every $ i < k$,
\begin{align}
\forall \z \in \mK_{\delta, \delta'/r}:\quad \Vert \y_{i+1} -\z \Vert^2 \leq \Vert \y_{i} -\z \Vert^2 - \delta^2 (r-\delta')^2, \label{eq:update_hyperplane_foo}
\end{align}
Specifically for $i=k-1$, and unrolling the recursion, we obtain  that for all $\z \in \mK_{\delta,\delta'/r}$, $\Vert \y - \z \Vert^2 \leq  \Vert \y_{1} - \z \Vert^2$, and since $\y_1$ is the projection of $\y_0$ onto $R\ball $ and $\mK_{\delta,\delta'/r} \subseteq R\ball$, it holds that for all $\z \in \mK_{\delta,\delta'/r}$,  $\Vert \y_{1} - \z \Vert^2 \leq  \Vert \y_0 - \z \Vert^2$, and we can conclude that indeed  for all $\z \in \mK_{\delta,\delta'/r}$, $\Vert \y - \z \Vert^2 \leq \Vert \y_{0} - \z \Vert^2 $, as needed.

 Now, we  upper-bound $k$ --- the number of iterations until Algorithm \ref{alg:CIP-SO} stops. Denote $\x_{i}^* = \argmin_{\x \in \mK_{\delta,\delta'/r}} \Vert \y_{i} - \x \Vert^2$. Using Eq. \eqref{eq:update_hyperplane_foo} for every iteration $i < k$ it holds that,
\begin{align*}
    \dist^2(\y_{i+1}, \mK_{\delta,\delta'/r})  & = \Vert  \y_{i+1} - \x_{i+1}^* \Vert^2 \leq \Vert \y_{i+1} - \x_{i}^* \Vert^2   \\ 
    & \leq \Vert \y_{i} - \x_{i}^* \Vert^2 -  \delta^2 (r-\delta')^2  = \dist^2(\y_{i}, \mK_{\delta,\delta'/r}) -  \delta^2 (r-\delta')^2.
\end{align*}
Unrolling the recursion, and Since $\y_1$ is the projection of $\y_0$ onto $R\ball $ and $\mK_{\delta, \delta'/r}  \subseteq R\ball$, we have
\begin{align*}
    \dist^2(\y, \mK_{\delta,\delta'/r}) & \leq \dist^2(\y_1, \mK_{\delta,\delta'/r}) - (k-1) \delta^2 (r-\delta')^2 \\
&  \leq \dist^2(\y_0, \mK_{\delta,\delta'/r}) - (k-1) \delta^2 (r-\delta')^2.  
\end{align*}
Thus, after at most 
\begin{align*}
    k = \frac{ \dist^2(\y_0, \mK_{\delta,\delta'/r}) - \dist^2(\y, \mK_{\delta,\delta'/r})}{ \delta^2 (r-\delta')^2}+1    
\end{align*}
iterations Algorithm \ref{alg:CIP-SO} must stop.
\end{proof} 

\subsection{SO-based algorithm for the full-information setting}

Our SO-based algorithm for the full-information setting, Algorithm  \ref{alg:OGD-SGO}, is given below.
\begin{algorithm}[!ht]
\KwData{horizon $T$, feasible set $\mK$, update step $ \eta$, squeeze parameter $ \delta$.}
$\Tilde{\y}_1 \gets \vz \in \mK_{\delta}$.\\
\For{$~ t = 1,\ldots, T ~$}{
    Play $\Tilde{\y}_{t} $ and observe $f_{t}(\Tilde{\y}_{t})$.\\
    Set $\nabla_t \in  \partial  f_t(\Tilde{\y}_{t})$ and update $\y_{t+1} = \Tilde{\y}_{t} - \eta \nabla_t$.\\
    Set $\Tilde{\y}_{t+1}  \gets$ Outputs of Algorithm \ref{alg:CIP-SO} with set $\mK$, radius $r$, initial vector $\y_{t+1}$, and squeeze parameters $( \delta, 0)$.
}
\caption{Online gradient descent via a separation oracle (SO-OGD)}\label{alg:OGD-SGO}
\end{algorithm}
 \begin{theorem}\label{thm:OGD-SGO}   
Suppose Assumption \ref{ass:bandit} holds. Fix $c >0$ such that $\delta= cT^{-\frac{1}{2}}\in(0,1)$, and set $\eta = \frac{c_1 r }{ 2 G_f }T^{-\frac{1}{2}}$. Algorithm \ref{alg:OGD-SGO} guarantees that the adaptive regret is upper bounded by
\begin{align*}
    \sup_{I=[s,e]\subseteq[T]} \bigg{\{} \sum_{t=s}^{e} f_t(\Tilde{\y}_{t}) - \min\limits_{\x_I \in \mK} \sum_{t=r}^{s} f_t(\x_I) \bigg{\}} 
    \leq &\left( G_f R c + \frac{r G_f }{4} + \frac{ 4R^2 G_f}{ r} \right) \sqrt{T},  
\end{align*}
and that the overall number of calls to the SO is upper bounded by  
\begin{align*}
    N_{calls} & \leq \left( \frac{ R  }{ r c} +  \frac{1 }{ 4 c^2}  + 1 \right) T .
\end{align*}

In particular, if $\frac{4 R}{r} \leq \sqrt{T}$, then setting $c=\frac{4 R}{r}$, we have that
\begin{align*}
    \sup\limits_{[s,e]\subseteq[T]} \bigg{\{} \sum_{t=s}^{e} f_t(\Tilde{\y}_{t}) - \min\limits_{\x_I \in \mK} \sum_{t=r}^{s} f_t(\x_I) \bigg{\}} \leq  G_f \left( \frac{r }{4} + \frac{ 8R^2}{ r} \right) \sqrt{T} , 
\end{align*}
and  
\begin{align*}
    N_{calls} & \leq \left(  \frac{5}{ 4  } +  \frac{r^2  }{ 64 R^2}  \right) T.
\end{align*}
\end{theorem}

 Before proving the theorem we need an additional observation.
\begin{observation} \label{obs:feasibility_to_dist}
Fix $\delta\in(0,1)$. For any $\y \in \mK$ it holds that $\dist(\y,\mK_\delta) \leq R\delta$.
\end{observation}

\begin{proof} [Proof of Theorem \ref{thm:OGD-SGO}]
First, we note that since for every $t \in [2,T]$, $\Tilde{\y}_{t}$ is output of Algorithm \ref{alg:CIP-SO}, using Lemma \ref{lemma:CIP-SO} with $\delta' = 0$, it follows that $\Tilde{\y}_{t} \in \mK $, and thus, Algorithm \ref{alg:OGD-SGO} indeed plays feasible points. Now, we prove the upper-bound on the adaptive regret. Fix an interval $I=[s,e], 1\leq s \leq e \leq T$, and a feasible minimizer w.r.t. this interval, $\x_I^* \in \argmin_{\x \in \mK} \sum_{t=s}^{e} f_t(\x)$. Define $\Tilde{\x}_{I} = (1-\delta) \x_I^* \in \mK_{\delta}$. Since for every $t \in [T]$, $ f_t(\cdot)$ is $G_f-$Lipschitz over $\mK$, we have that
\begin{align*}
    \sum_{t=s}^{e} f_t\left( \Tilde{\y}_{t} \right) - f_t(\x_I^*) & = \sum_{t=s}^{e} f_t \left( \Tilde{\y}_{t} \right) - f_t \left( \Tilde{\x}_{I} \right) + f_t \left( \Tilde{\x}_{I} \right) - f_t(\x_I^*) \\
& \leq  G_f R \delta T  + \sum_{t=s}^{e} f_t \left( \Tilde{\y}_{t} \right) - f_t \left( \Tilde{\x}_{I} \right). 
\end{align*}
Using Lemma \ref{lemma:CIP-SO} with $\delta' = 0$ for all $t\geq 1$ we have that, $\Tilde{\y}_{t} \in \mK$ is an infeasible projection of $\y_{t}$ over $\mK_{\delta}$. Thus, from Lemma \ref{lemma:OGD-WF}, we have that
\begin{align*}
    \sum_{t=s}^{e} f_t(\Tilde{\y}_{t}) - \sum_{t=s}^{e} f_t(\Tilde{\x}_I) \leq \frac{ \left\Vert   \Tilde{\y}_s - \x \right\Vert^2 }{2\eta} + \frac{\eta  }{2}  \sum_{s=1}^{e} \Vert \nabla_t \Vert^2 \leq \frac{ 2R^2 }{\eta} + \frac{\eta G_f^2 }{2} T .
\end{align*}
Combining the last two equations, we obtain that
\begin{align*}
    \sum_{t=s}^{e} f_t\left( \Tilde{\y}_{t} \right) - f_t(\x_I^*)  \leq \left( G_f R \delta + \frac{G_f^2 \eta}{2} \right) T + \frac{ 2R^2}{\eta}.
\end{align*}
The regret bound in the theorem now follows from plugging-in the values of $\delta,\eta$ listed in the theorem.

We turn to upper-bound the number of calls to the SO. For every $t\geq 1$, denote $\Tilde{\x}_{t}^* = \argmin_{\x \in \mK_{\delta}} \Vert \x - \Tilde{\y}_{t} \Vert$. Since $\y_{t+1} = \Tilde{\y}_{t} - \eta  \nabla_t$, we have
\begin{align*}
    \dist(\y_{t+1}, \mK_{\delta})  \leq \Vert \Tilde{\x}_{t}^* - \y_{t+1} \Vert  & \leq \Vert \Tilde{\x}_{t}^* - \Tilde{\y}_{t} \Vert + \Vert \Tilde{\y}_{t} - \y_{t+1} \Vert  \leq \dist(\Tilde{\y}_{t}, \mK_{\delta}) + \Vert \eta  \nabla_t \Vert. 
\end{align*}
It follows that, for any iteration $t\geq 1$, Algorithm \ref{alg:OGD-SGO} calls Algorithm \ref{alg:CIP-SO} with $\y_{t+1}$ such that
\begin{align} \label{eq:initial_gap_OGD-SGO}
    \dist^2(\y_{t+1}, \mK_{\delta}) \leq  \dist^2(\Tilde{\y}_{t}, \mK_{\delta}) + 2 \dist(\Tilde{\y}_{t}, \mK_{\delta}) \eta G_f + \eta^2 G_f^2. 
\end{align}
Using Lemma \ref{lemma:CIP-SO} with initial point $\y_{t+1}$, feasible set $\mK $, radius $r$, squeeze parameters $(\delta,\delta' = 0)$, and the returned point $\Tilde{\y}_{t+1}$, we have that for every iteration $t\geq 1$, Algorithm \ref{alg:CIP-SO} makes at most 
\begin{align*}
    \frac{\dist^2 \left(\y_{t+1}, \mK_{\delta} \right) - \dist^2 \left(\Tilde{\y}_{t+1}, \mK_{\delta} \right)}{ \delta^2 r^2} +1
\end{align*} 
iterations. Thus, using Eq.\eqref{eq:initial_gap_OGD-SGO} and Observation \ref{obs:feasibility_to_dist}, the overall number of calls to the SO of $\mK$ that Algorithm \ref{alg:CIP-SO} makes is
\begin{align*}
    N_{calls} & \leq  \sum_{t=1}^{T} \frac{1 }{ \delta^2 r^2} \left(\dist^2(\Tilde{\y}_{t}, \mK_{\delta}) + 2 R \delta \eta G_f + \eta^2 G_f^2 - \dist^2(\Tilde{\y}_{t+1}, \mK_{\delta}) \right) +1  \\
    & \leq \frac{2 R G_f  }{ r^2} \frac{\eta }{ \delta} T + \frac{G_f^2  }{ r^2} \frac{\eta^2 }{ \delta^2}  T +T,
\end{align*}
where the last inequality is since $\dist^2(\Tilde{\y}_{1}, \mK_{\delta}) = 0$.
\end{proof}

\subsection{SO-based algorithm for the bandit setting}
Similarly to our LLO-based algorithm for the bandit setting,  our SO-based bandit algorithm follows from combining our SO-based algorithm for the full-information setting together with the use of unbiased estimators for the gradients of smoothed versions of the loss functions, as pioneered in  \cite{Flaxman05}. Our algorithm for the bandit feedback, Algorithm \ref{alg:BGD-SGO}, is given below. As opposed to the full-information setting which used a single squeeze parameter (i.e., we set $\delta' =0$ when considering the squeezed set $\mK_{\delta,\delta'/r}$), in the bandit setting, due to the ball-sampling technique which is used to construct the unbiased gradient estimators, in order to keep the iterates feasible, we set $\delta'$ to be strictly positive.

\begin{algorithm}[!ht]
\KwData{horizon $T$, feasible set $\mK$ with parameters $r,R$, update step $\eta$, squeeze parameters $(\delta,\delta')$.} 
$\Tilde{\y}_1 \gets \vz \in \mK_{\delta,\delta'}$ \\
\For{$~ t = 1,\ldots, T ~$}{
    Set $\u_t$ $\sim S^n$, play $\z_t = \Tilde{\y}_{t}+ \delta' \u_t $, and observe $f_{t}(\z_t)$.\\
    Set $\g_t = \frac{n}{\delta'} f_t(\z_{t}) \u_t$ and update $\y_{t+1} = \Tilde{\y}_{t} - \eta \g_t$. \\ 
    Set $\Tilde{\y}_{t+1} \gets$ output of Algorithm \ref{alg:CIP-SO} with set $\mK$, radius $r$, initial vector $\y_{t+1}$, and squeeze parameters $ ( \delta , \delta')$.
    }
\caption{Bandit online gradient descent via a separation oracle (SO-BGD)}\label{alg:BGD-SGO}
\end{algorithm}

\begin{theorem} \label{thm:BGD-SGO}
Suppose Assumption \ref{ass:bandit} holds. Fix some $c', c>0$ such that $ 2 c' T^{-1/4} <r $ and $ c T^{-1/4} < 1$. Setting $\eta = \frac{r}{4\sqrt{nM}}T^{-\frac{3}{4}}, \delta = c T^{-1/4} , \delta' = c' T^{-\frac{1}{4}}$ in Algorithm \ref{alg:BGD-SGO}, guarantees that the adaptive expected regret is upper bounded as follows
\begin{align*}
    &AER_T = \sup\limits_{I=[s,e]\subseteq[T]} \Bigg{\{} \E \left[ \sum_{t=s}^{e} f_t(\z_{t}) \right] - \min\limits_{\x_I \in \mK}  \sum_{t=s}^{e} f_t(\x_I)  \Bigg{\}}  \leq \\
    & ~~~~~~~~~~~~~~~~~ \leq G_f R \left( \frac{3c'}{R} + \frac{c'}{r} +  c  + \frac{4  \sqrt{nM}}{r G_f} + \frac{(nM)^\frac{3}{2}}{8 G_f R} \frac{r}{{c'}^2} \right) T^{\frac{3}{4}}  + G_f R\frac{cc'}{r} T^{\frac{1}{2}},
\end{align*}
and the overall number of calls to SO is upper bounded by 
\begin{align*}
    N_{calls} & \leq  T + \frac{2  R  \sqrt{nM}  }{  r} \frac{1}{c c'} T^{\frac{3}{4}}+ \frac{nM }{ 4 }  \frac{1}{c^2 {c'}^2} T^{\frac{1}{2}} .
\end{align*}
In particular, if $T^{1/4} > \max \{ \frac{2\sqrt{nM}}{r}, \frac{8}{r} \} $, then setting $c=\frac{8}{r} $ and $c' = \sqrt{nM}$, we have
\begin{align*}
    AER_T   \leq R \sqrt{nM} \left( \frac{4G_f}{r}   + \frac{4  }{r} + \frac{r }{8 R} \right) T^{\frac{3}{4}}
    & +  \frac{8G_f R}{r} T^{\frac{3}{4}}   + G_f R\frac{8\sqrt{nM}}{r^2} T^{\frac{1}{2}},
\end{align*}
and
\begin{align*}
    N_{calls} & \leq  T + \frac{  R  }{ 4 } T^{\frac{3}{4}}+  \frac{r^2}{ 256} T^{\frac{1}{2}} .
\end{align*}
\end{theorem}

\begin{proof}
First, we establish that Algorithm \ref{alg:BGD-SGO} indeed plays feasible points, meaning  $\z_t \in \mK$ for all $t\in[T]$. Since for every $t\geq 1$, Algorithm \ref{alg:CIP-SO} returns $\Tilde{\y}_{t} \in \mK_{\delta'/r} = (1-\delta'/r) \mK$ and $r\ball \subseteq \mK$, it follows that indeed $\z_t = \Tilde{\y}_{t} + \delta' \u_t \in \mK$ for every $t \in [T]$. 

Now, we turn to prove the upper-bound on the adaptive expected regret. Let us fix some interval $I=[s,e], 1 \leq s \leq e \leq T$. We start with an upper bound on $\E \left[ \sum_{t=s}^{e} \widehat{f}_{t,\delta'} \left(\Tilde{\y}_{t}\right) - \widehat{f}_{t,\delta'} (\x) \right]$ for every $\x \in \mK_{\delta,\delta'/r} = (1-\delta'/r)(1-\delta)\mK$. We will first take a few preliminary steps. For every $t \in \left[T\right]$, denote by  $\mathcal{F}_t = \{ \Tilde{\y}_1,  \dots, \Tilde{\y}_{t-1}, \g_1, \dots, \g_{t-1} \}$ the history of all predictions and gradient estimates up to time $t$. Since for all $t\geq 1$, $\g_t$ is an unbiased estimator of $\nabla {\widehat{f}}_{t,\delta'} \left(\tilde{\y}_{t}\right) $, i.e., $\E \left[ \g_t |\mathcal{F}_t\right] = \nabla {\widehat{f}}_{t,\delta'} \left(\tilde{\y}_{t}\right)$,  we have that for all $t \in \left[T\right]$ and $\x \in \mK_{\delta,\delta'/r}$, it holds that
\begin{align}
    \E \left[ \g_t^{\top} \left(\Tilde{\y}_{t} - \x\right) \right] & = \E \left[ \E \left[  \g_t | \mathcal{F}_{t} \right] ^{\top}   (\Tilde{\y}_{t} -   \x) \right] =  \E \left[\nabla \widehat{f}_{t,\delta'} (\Tilde{\y}_{t})^\top (\Tilde{\y}_{t} -   \x) \right]. \label{eq:bf_unbiased_gradient_esimator}
\end{align}
From Lemma \ref{lemma:CIP-SO} with $\delta' \neq 0$, we have that for every $t \in [T]$, the point $\Tilde{\y}_{t} \in \mK_{\delta'/r}$, and is an infeasible projection of $\y_{t}$ over $\mK_{\delta,\delta'/r}$. Thus, we have that
\begin{align*}
    \forall t\in[T]~\forall \x \in \mK_{\delta,\delta'/r} : ~~\Vert \Tilde{\y}_{t+1}  - \x \Vert^2  \leq \Vert  \y_{t+1} - \x \Vert^2.
\end{align*}
Since $\y_{t+1} = \Tilde{\y}_{t} -  \eta \g_t$, for every $t\in[T]$ and $\x \in \mK_{\delta,\delta'/r}$ we have that
\begin{align*}
    \Vert \Tilde{\y}_{t+1}  - \x \Vert^2  \leq  \left\Vert  \Tilde{\y}_{t} -  \eta  \g_t  -\x \right\Vert^2 = \left\Vert   \Tilde{\y}_{t} - \x \right\Vert^2  +   \eta^2 \left\Vert \g_t \right\Vert^2 - 2 \eta   \g_t^\top (\Tilde{\y}_{t} - \x).
\end{align*}
Rearranging, we obtain that for every $t\in[T]$ and  $\x \in \mK_{\delta,\delta'/r}$, it holds that 
\begin{align*}
    \g_t^\top  (\Tilde{\y}_{t} - \x)  \leq  \frac{ \left\Vert   \Tilde{\y}_{t} - \x \right\Vert^2}{2\eta} - \frac{ \Vert \Tilde{\y}_{t+1} - \x \Vert^2}{2\eta} + \frac{\eta}{2} \left\Vert  \g_t \right\Vert^2.
\end{align*}
Summing over the interval $[s,e]$ and taking expectation,  we have that
\begin{align*}
    \E \left[ \sum\limits_{t =s}^{e}  \g_t^\top \left(\Tilde{\y}_{t} - \x\right) \right] & \leq  \E \left[  \sum\limits_{t =s}^{e}  \frac{ \left\Vert   \Tilde{\y}_{t} - \x \right\Vert^2}{2\eta} - \frac{ \Vert \Tilde{\y}_{t+1} - \x \Vert^2}{2\eta} \right] + \frac{\eta  }{2}   \sum\limits_{t =s}^{e} \E \left[ \left\Vert  \g_t \right\Vert^2 \right].
\end{align*}
Since $\Tilde{\y}_t \in \mK_{\delta'/r} $ for every $t \in [T]$, then $\Vert \Tilde{\y}_t - \x \Vert \leq 2R$ for every $\x \in \mK_{\delta,\delta'/r}$, and thus,
\begin{align*}
    \E \left[ \sum\limits_{t =s}^{e}  \g_t^\top \left(\Tilde{\y}_{t} - \x\right) \right] \leq &  \frac{R}{\eta} + \frac{\eta  }{2}   \sum\limits_{t =s}^{e} \E \left[ \left\Vert  \g_t \right\Vert^2 \right].
\end{align*}
Using Eq. \eqref{eq:bf_unbiased_gradient_esimator}, for every $\x \in \mK_{\delta,\delta'/r}$ we have that,
\begin{align*}
    \sum_{t=s}^{e} \E \left[ \nabla \widehat{f}_{t,\delta'} \left(\Tilde{\y}_{t}\right)^\top \left(\Tilde{\y}_{t} - \x\right) \right]  = \sum_{t=s}^{e} \E \left[ \g_t^{\top}  \left(\Tilde{\y}_{t} - \x\right) \right] \leq \frac{R}{\eta} + \frac{\eta  }{2}   \sum\limits_{t =s}^{e} \E \left[ \left\Vert  \g_t \right\Vert^2 \right]. 
\end{align*}
Since for every $t \in [T] $, $ f_t (\cdot) $ is convex in $\mK$, using Lemma \ref{lemma:hazan_smooth}, it holds that  $\widehat{f}_{t,\delta'} (\cdot) $ is convex in $\mK_{\delta'/r}$. Thus,  for every $\x \in  \mK_{\delta,\delta'/r}$ we obtain that,
\begin{align}
    \E \left[ \sum_{t=s}^{e} \widehat{f}_{t,\delta'} \left(\Tilde{\y}_{t}\right) - \widehat{f}_{t,\delta'} (\x) \right]  \leq \frac{R}{\eta} + \frac{\eta  }{2}   \sum\limits_{t =s}^{e} \E \left[ \left\Vert  \g_t \right\Vert^2 \right]. \label{eq:bf_regret_infeasible_projection}
\end{align}

Let us denote by $\x_I^*$ a feasible minimizer w.r.t. to interval $I=[s,e]$, i.e., $\x_I^* \in \argmin_{\x \in \mK} \sum_{t=s}^{e} f_t(\x)$, and define accordingly $\Tilde{\x}_{I}^*= (1-\delta'/r)(1-\delta) \x_I^* \in  \mK_{\delta,\delta'/r} $. It holds that,
{\small \begin{align}
    \E \left[\sum_{t=s}^{e} f_t(\z_{t}) \right] -   \sum_{t=s}^{e} f_t(\x_I)   = & \E \left[ \sum_{t=s}^{e} f_t(\z_{t}) - f_t\left(\Tilde{\y}_{t}\right)  +  \sum_{t=s}^{e} f_t\left(\Tilde{\y}_{t}\right) - f_t(\Tilde{\x}_{I}^*) \right]   +  \sum_{t=s}^{e} f_t(\Tilde{\x}_{I}^*) - \sum_{t=s}^{e} f_t(\x_I^*). \label{eq:bf_full_regret_AOGD}
\end{align}}
Since for every $t \in [T]$ $f_t(\cdot)$ is $G_f$-Lipschitz, we have that
\begin{align*}
    \E \left[ \sum_{t=s}^{e} f_t(\z_{t}) - f_t(\Tilde{\y}_{t}) \right] & =  \sum_{t=s}^{e} \E \left[ f_t(\Tilde{\y}_{t} + \delta' \u_t) - f_t(\Tilde{\y}_{t}) \right] \leq G_f \delta' T,
\end{align*}
and
\begin{align*}
    \sum_{t=s}^{e} f_t(\Tilde{\x}_{I}^*) - f_t(\x_I^*) \leq \sum_{t=s}^{e} \nabla f_t(\Tilde{\x}_{I}^*)^\top \left( \Tilde{\x}_{I}^* - \x_{I}^* \right) & \leq \sum_{t=s}^{e} G_f  \Vert \Tilde{\x}_{I}^* - \x_I^* \Vert  \leq G_f R \left( \frac{\delta'}{r} + \delta + \frac{\delta\delta'}{r} \right) T.
\end{align*}
Using Lemma \ref{lemma:hazan_smooth} and Eq. \eqref{eq:bf_regret_infeasible_projection}, we have
\begin{align*}
    \E \left[ \sum_{t=s}^{e} f_t\left(\Tilde{\y}_{t}\right) - f_t(\Tilde{\x}_{I}^*) \right] = & \E \left[ \sum_{t=s}^{e} f_t\left(\Tilde{\y}_{t}\right) - \widehat{f}_{t,\delta'}\left(\Tilde{\y}_{t}\right) \right]  + \E \left[ \sum_{t=s}^{e} \widehat{f}_{t,\delta'}(\Tilde{\x}_{I}^*) - f_t(\Tilde{\x}_{I}^*) \right]  \\
    &  + \E \left[ \sum_{t=s}^{e} \widehat{f}_{t,\delta'} \left(\Tilde{\y}_{t}\right) - \widehat{f}_{t,\delta'} (\Tilde{\x}_{I}^*) \right] \leq 2 \delta' G_f T + \frac{R}{\eta} + \frac{\eta  }{2}   \sum\limits_{t =s}^{e} \E \left[ \left\Vert  \g_t \right\Vert^2 \right].
\end{align*}
Combining the last three equations and Eq. \eqref{eq:bf_full_regret_AOGD}, and using the fact that $\left\Vert \g_t \right\Vert \leq \frac{nM}{\delta'}$, we obtain that
\begin{align*}
    \E \left[\sum_{t=s}^{e} f_t(\z_{t})  -  f_t(\x_I^*) \right] \leq G_f \left(3\delta' + R \left( \frac{\delta'}{r} + \delta + \frac{\delta\delta'}{r} \right) \right) T + \frac{R}{\eta} + \frac{n^2M^2}{2} \frac{\eta}{{\delta'}^2} T.
\end{align*}
Plugging-in the values of $\eta,\delta,\delta'$ listed in the theorem, we obtain the adaptive expected regret bound in the theorem.

We now move on to upper-bound the overall number of calls to the SO of $\mK$. For every $t\in[T]$, let us denote $\x_{t}^* = \argmin_{\x \in \mK_{\delta,\delta'/r}} \Vert \x - \Tilde{\y}_{t} \Vert$. Since Algorithm \ref{alg:BGD-SGO} updates $\y_{t+1} = \Tilde{\y}_{t} - \eta \g_t$, we have
\begin{align*}
    \dist \left( \y_{t+1}, \mK_{\delta,\delta'/r} \right) \leq \Vert \Tilde{\x}_{t}^* - \y_{t+1} \Vert & \leq \Vert \Tilde{\x}_{t}^* - \Tilde{\y}_{t} \Vert + \Vert \Tilde{\y}_{t} - \y_{t+1} \Vert = \dist \left( \Tilde{\y}_{t}, \mK_{\delta,\delta'/r} \right) + \Vert \eta \g_t \Vert,
\end{align*}
which, by plugging-in the upper-bound on $\Vert{\g_t}\Vert$, gives
\begin{align}
    \dist^2 \left( \y_{t+1}, \mK_{\delta,\delta'/r} \right) \leq & ~ \dist^2 \left( \Tilde{\y}_{t}, \mK_{\delta,\delta'/r} \right) + 2 \dist \left( \Tilde{\y}_{t}, \mK_{\delta,\delta'/r} \right) \eta \frac{nM}{\delta'} + \eta^2 \frac{(nM)^2}{{\delta'}^2}. \label{eq:initial_gap_OGD-SGO_b}
\end{align}
For any $t\in[T]$, using Lemma \ref{lemma:CIP-SO} with initial point $\y_{t+1}$, feasible set $\mK$, radius $r$, squeeze parameters $(\delta, \delta')$, and the returned point $\Tilde{\y}_{t+1}$, we have that Algorithm \ref{alg:CIP-SO} makes at most
{\small\begin{align*}
    \frac{\dist^2 \left( \y_{t+1}, \mK_{\delta,\delta'/r} \right) - \dist^2 \left( \Tilde{\y}_{t+1}, \mK_{\delta,\delta'/r} \right)}{ \delta^2 \left( r - \delta'  \right)^2}+1
\end{align*}}
iterations. Since  $\tilde{\y}_t \in \mK_{\delta'/r} \subseteq \mK$ and $(1-\delta) \mK_{\delta'/r} = \mK_{\delta,\delta'/r}$, using Observation \ref{obs:feasibility_to_dist} it holds that $\dist \left( \Tilde{\y}_{t}, \mK_{\delta,\delta'/r} \right) \leq R \delta$. Thus, using this observation and Eq.\eqref{eq:initial_gap_OGD-SGO_b}, the overall number of calls to the SO of $\mK$ that Algorithm \ref{alg:CIP-SO} makes is
{\small \begin{align*}
    N_{calls} & \leq T + \frac{1 }{ \delta^2 \left( r - \delta' \right)^2} \sum_{t=1}^{T}  \left( \dist^2 \left( \Tilde{\y}_{t}, \mK_{\delta,\delta'/r} \right) + 2 R \delta \eta \frac{nM}{\delta'} + \eta^2 \frac{(nM)^2}{{\delta'}^2} - \dist^2 \left( \Tilde{\y}_{t+1}, \mK_{\delta,\delta'/r} \right)\right) \\
    & \leq  \left( 1+ \frac{8  R  nM  }{  r^2} \frac{\eta}{\delta \delta'} + \frac{4(nM)^2 }{  r^2}  \frac{\eta^2}{\delta^2 {\delta'}^2}\right) T,
\end{align*}}
where last inequality follows since $\dist^2 \left( \Tilde{\y}_{1}, \mK_{\delta,\delta'/r} \right) = 0$, and $\delta' \leq r/2$.
\end{proof}

\bibliographystyle{plain} 
\bibliography{bib}

\appendix

\section{Proof of Lemma \ref{lemma:OGD-WF}}
\begin{proof}
Fix some iteration $t$ of Algorithm \ref{alg:OGD-WF}. Since $\Tilde{\y}_{t+1}$ is an infeasible projection of  $\y_{t+1}$,  and $\y_{t+1} = \Tilde{\y}_t -  \eta_t  \nabla_t$, we have that
\begin{align*}
  \forall\x\in\mK: ~  \Vert \Tilde{\y}_{t+1}  - \x \Vert^2  & \leq \Vert  \y_{t+1} - \x \Vert^2 = \left\Vert  \Tilde{\y}_t -  \eta_t \nabla_t -\x \right\Vert^2 \\
  & \leq \left\Vert   \Tilde{\y}_t - \x \right\Vert^2  +  \eta_t^2 \Vert \nabla_t \Vert^2 - 2 \eta_t \nabla_t^\top (\Tilde{\y}_{t} - \x).
\end{align*}
Rearranging, then we have 
\begin{align*}
   \forall\x\in\mK: ~  \nabla_t^\top & (\Tilde{\y}_{t} - \x)  \leq  \frac{ \left\Vert   \Tilde{\y}_t - \x \right\Vert^2}{2\eta_t} - \frac{ \Vert \Tilde{\y}_{t+1} - \x \Vert^2}{2\eta_t} + \frac{\eta_t  \Vert \nabla_t \Vert^2 }{2}.
\end{align*}
Fix some positive integers $1\leq s \leq e \leq T$. Summing over the interval $[s,e]$, we have that
\begin{align}
    \forall\x\in\mK: ~ \sum\limits_{t =s}^{e} \nabla_t^\top ( \Tilde{\y}_t - \x) & \leq \frac{ \left\Vert   \Tilde{\y}_s - \x \right\Vert^2}{2\eta_s} + \sum\limits_{t = s+1 }^{e}  \left( \frac{ 1}{2\eta_t} - \frac{ 1}{2\eta_{t-1}}\right)  \left\Vert   \Tilde{\y}_t - \x \right\Vert^2  + \sum\limits_{t =s}^{e}  \frac{\eta_t  }{2} \Vert \nabla_t \Vert^2 . \label{eq:ogd_wf_gradients_regret}
\end{align}
Using the convexity of each $f_t(\cdot)$ and plugging-in $\eta_t = \eta$ for all $t\geq 1$, we have that
\begin{align*}
  \forall\x\in\mK: ~   \sum\limits_{t =s}^{e}  f_t(\Tilde{\y}_{t}) -  f_t(\x) \leq & \frac{ \left\Vert   \Tilde{\y}_s - \x \right\Vert^2 }{2\eta} + \frac{\eta  }{2}  \sum\limits_{t =s}^{e}  \Vert \nabla_t \Vert^2,
\end{align*}
which yields the first guarantee of the lemma.

In case all loss function $f_t(\cdot), 1\leq t\leq T$, are $\alpha$-strongly convex, using the inequality $f_t(\y) - f_t(\x) \leq \nabla f_t (\y)^\top (\y - \x) - \frac{\alpha}{2} \Vert \y - \x \Vert^2, (\x,\y)\in\reals^n\times\reals^n$, and setting $(s,e) = (1,T)$ in Eq.\eqref{eq:ogd_wf_gradients_regret}, for every $\x \in \mK$ we have that,
\begin{align*}
    \sum\limits_{t =1}^{T}  f_t(\Tilde{\y}_{t}) -  f_t(\x) \leq & \sum\limits_{t =1}^{T}  \frac{\eta_t \Vert \nabla_t \Vert^2 }{2}  + \left( \frac{ 1 }{2\eta_1} - \frac{\alpha}{2} \right) \left\Vert   \Tilde{\y}_1 - \x \right\Vert^2 \\
& + \sum\limits_{t = 2 }^{T}  \left( \frac{ 1}{2\eta_t} - \frac{ 1}{2\eta_{t-1}} - \frac{\alpha}{2}\right)  \left\Vert   \Tilde{\y}_t - \x \right\Vert^2. 
\end{align*}
Plugging in $\eta_t = \frac{1}{\alpha{}t}$, we obtain the second guarantee of the lemma. 
\end{proof}

\section{ Proofs of Additional  Observations}

\begin{proof} [Proof of Observation \ref{obs:radius_ball_in_k_delta}]
First we prove that $(r-\delta')\ball\subseteq\mK_{\delta'/r} = (1-\delta'/r)\mK$.
Fix some $\u \in (r-\delta')\ball$, i.e., $\Vert{\u}\Vert\leq r-\delta'$. Since $r\ball \subseteq \mK$, it holds that $\u \frac{1}{\left( 1 - \delta'/r \right)} = \u \frac{r}{r-\delta'} \in r\ball \subseteq \mK$. This in turn implies that $\u=(1-\delta'/r)\u\frac{1}{1-\delta'/r}\in(1-\delta'/r)\mK=\mK_{\delta'/r}$.

Now, we recall that if a convex set $\mP\subset\reals^n$ satisfies that $p\ball\subseteq\mP$, for some $p>0$, then for any $\gamma\in[0,p]$ and any $\z\in(1-\gamma/p)\mP$, it holds that $\z+\gamma\ball\subseteq\mP$ (see for instance the chapter on bandit algorithms in \cite{HazanBook}. Applying this with $\mP = \mK_{\delta'/r}$, $p=(r-\delta')$, and $\gamma = \delta(r-\delta')$, we have that for any $\z\in(1- \delta(r-\delta')/(r-\delta'))\mK_{\delta'/r} = (1-\delta)\mK_{\delta'/r}=\mK_{\delta,\delta'/r}$, it holds that $\z + \delta(r-\delta')\ball\subseteq\mK_{\delta'/r}$, as needed. 
\end{proof}

\begin{proof} [Proof of Observation \ref{obs:feasibility_to_dist}]
Denote $\x^* = \argmin\limits_{\x \in \mK_\delta} \Vert \x - \y \Vert^2$ and $\y_{\delta} = (1-\delta) \y$. Since $\y \in \mK$ and $\y_{\delta} \in \mK_{\delta}$ it holds that 
\begin{align*}
    \dist(\y,\mK_{\delta}) = \Vert  \x^* - \y \Vert \leq \Vert  \y_{\delta} - \y \Vert  = \Vert  \delta \y \Vert  \leq \delta R.
\end{align*}
\end{proof}

\end{document}